\documentclass{article}
\usepackage{microtype}
\usepackage{graphicx}
\usepackage{subcaption}
\usepackage{booktabs} 
\usepackage{enumitem}

\usepackage{hyperref}

\usepackage[preprint]{icml2026}

\usepackage{amsmath}
\usepackage{amssymb}
\usepackage{mathtools}
\usepackage{amsthm}
\usepackage{bm}
\usepackage{macros}
\usepackage{nicefrac}
\usepackage{thmtools}
\usepackage{thm-restate}
\usepackage[section]{placeins}

\usepackage[capitalize,noabbrev]{cleveref}

\makeatletter
\newcommand*\bigcdot{\mathpalette\bigcdot@{.4}}
\newcommand*\bigcdot@[2]{\mathbin{\vcenter{\hbox{\scalebox{#2}{$\m@th#1\bullet$}}}}}
\makeatother

\newcommand{\mdot}[1]{\overset{\bigcdot}{#1}}
\newcommand{\twodot}[1]{\overset{\bigcdot\bigcdot}{#1}}
\newcommand{\qq}{\mathbf{q}}

\theoremstyle{plain}
\newtheorem{theorem}{Theorem}[section]

\newtheorem{lemma}[theorem]{Lemma}

\theoremstyle{definition}

\newtheorem{assumption}[theorem]{Assumption}
\theoremstyle{remark}

\newtheorem{property}[theorem]{Property}

\usepackage[textsize=tiny]{todonotes}

\hypersetup{
  breaklinks  = true,
  colorlinks   = true, 
  urlcolor     = blue, 
  linkcolor    = red, 
  citecolor   = blue 
}

\icmltitlerunning{Low-Entropy Outputs of Softmax}

\begin{document}

\twocolumn[
  \icmltitle{Gradient Flow Polarizes Softmax Outputs towards Low-Entropy Solutions}




  \icmlsetsymbol{equal}{*}

  \begin{icmlauthorlist}
      \icmlauthor{Aditya Varre}{yyy}
  \icmlauthor{Mark Rofin}{yyy}
  \icmlauthor{Nicolas Flammarion}{yyy}
  \end{icmlauthorlist}

  \icmlaffiliation{yyy}{Theory of Machine Learning Lab, EPFL}

  \icmlcorrespondingauthor{Aditya Varre}{aditya.varre@epfl.ch}

  \icmlkeywords{Machine Learning, ICML}

  \vskip 0.3in
]



\printAffiliationsAndNotice{}  

\begin{abstract}
Understanding the intricate non-convex training dynamics of softmax-based models is crucial for explaining the empirical success of transformers. In this article, we analyze the gradient flow dynamics of the \emph{value-softmax} model, defined as $\mathcal{L}(\mathbf{V} \sigma(\mathbf{a}))$, where $\mathbf{V}$ and $\mathbf{a}$ are a learnable value matrix and attention vector, respectively. As the \emph{matrix times softmax vector} parameterization constitutes the core building block of self-attention, our analysis provides direct insight into transformer's training dynamics. We reveal that gradient flow on this structure inherently drives the optimization toward solutions characterized by low-entropy outputs. We demonstrate the universality of this polarizing effect across various objectives, including logistic and square loss. Furthermore, we discuss the practical implications of these theoretical results, offering a formal mechanism for empirical phenomena such as \emph{attention sinks} and \emph{massive activations}.
\end{abstract}

\section{Introduction}

Despite the widespread adoption of large language models (LLMs), our understanding of the internal mechanisms by which they process information remains limited. In particular, the self-attention mechanism and the representations it induces play a central role in model behavior, yet many aspects of their functioning remain in the early stages of understanding. Recent work has begun to shed light on this opacity by identifying recurring learned structures within transformer architectures, such as induction heads~\citep{olsson2022context} and attention sinks~\citep{xiao2024efficient}. These components have attracted significant attention as crucial building blocks underlying complex model capabilities.

Building on these efforts, identified components are often examined in isolation to assess how architectural design choices, training algorithms and hyperparameters, as well as properties of the training data, contribute to their emergence. A theoretical understanding of these factors is a necessary step toward improving the robustness, efficiency, and safety of large language models. 
In this work, we pursue this objective by studying a broader phenomenon that underlies many of these components: \emph{low-entropy (sparse) attention patterns}, where the attention distribution concentrates most of its mass on a small number of tokens, and by asking whether this sparsification is solely enforced by the task.

\begin{figure}[t]
  \centering
  \setlength{\tabcolsep}{2pt}

  \begin{subfigure}[t]{0.31\columnwidth}
    \centering
    \includegraphics[width=\linewidth]{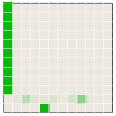}
    \caption{Softmax}
    \label{fig:attn-softmax}
  \end{subfigure}\hfill
  \begin{subfigure}[t]{0.31\columnwidth}
    \centering
    \includegraphics[width=\linewidth]{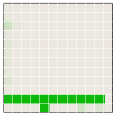}
    \caption{Sigmoid}
    \label{fig:attn-sigmoid}
  \end{subfigure}\hfill
  \begin{subfigure}[t]{0.31\columnwidth}
    \centering
    \includegraphics[width=\linewidth]{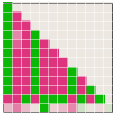}
    \caption{Linear}
    \label{fig:attn-linear}
  \end{subfigure}

  \caption{Representative attention patterns of the 2nd Transformer layer solving an induction task. Sigmoid and linear attentions are trained without additional normalization (see Table \ref{table:attention-types}). An attention sink clearly emerges for softmax but not other attentions.}
  \label{fig1}
  \vspace{-10pt}
\end{figure}



As a motivating example, we consider \emph{attention sinks}~\citep{xiao2024efficient}: sparse attention patterns characterized by attention mass that is highly concentrated on one token, often the first token in the sequence. A common hypothesis is that such low-entropy patterns are useful because concentrating mass on a small set of tokens induces a systematic bias in the attention output; however, the need for such bias is not always clear from task semantics or sequence structure. Moreover, some softmax-free attention types do not form attention sinks (Figure \ref{fig1}), pointing to the role of architecture in modulating their emergence.
This motivates us to ask the following question: 
\emph{Why does attention become sparse in the first place, and does low-entropy attention reflect a functional requirement of the task or an implicit preference induced by optimization and parameterization?}

We investigate this question by isolating the minimal computation underlying a single attention head. The output of a head can be written as a value matrix multiplied by a softmax attention distribution, i.e., $\VV\sf(\aa)$. In this form, an attention sink corresponds to $\sf(\aa)$ becoming low-entropy (near one-hot), concentrating mass on a single token. We therefore study a simplified \emph{value--softmax} model $\VV\sf(\aa)$, where both $\VV$ and $\aa$ are trainable, and analyze whether gradient-based training exhibits an \emph{implicit bias toward sparsification} even when many dense solutions exist. Concretely, we consider several losses with a fixed target vector, analyze gradient-flow dynamics, and complement the theory with empirical evidence.

Our contributions are:
\begin{itemize}[noitemsep, topsep=1pt]
\item We study the value-softmax model with logistic loss and show that among many possible solutions, the gradient flow is biased towards the sparse solution and the output of the softmax, i.e., $\sf(\aa)$ converges to a \emph{one-hot} vector. This sparsification is attributed to the polarizing dynamics of the gradient flow.
\item We extend these polarizing dynamics to the case of regression and attribute the sparsification to the speed of convergence. We further discuss how alternative nonlinearities and normalization schemes modulate (or prevent) this polarization.
%
\item We connect these results to transformer attention, arguing that attention sinks and massive activations can emerge as a consequence of the implicit bias induced by the softmax parameterization, helping explain sink formation.
\item We provide empirical evidence on how normalization, both its presence and absence can impact the proportion of attention sinks specifically in the formation of induction heads. We also discuss the pitfalls of such sparsification bias, where a dispropotionate burden of decision making is attributed to a single token.
\end{itemize}

\subsection{Related Work}

Our work lies at the intersection of mechanistic studies of attention in transformers, the emergence of low-entropy attention patterns, optimization-driven implicit bias in softmax-parameterized models, and polarization-style dynamics in gradient-flow.

\paragraph{Attention sinks and massive activations.}
\citet{xiao2024efficient} identified that large proportional of attention scores is focused on the first token and termed such tokens as attention sinks.~\citet{yu2024unveiling,sun2024massive} also identified that the attention sinks do not necessarily correspond to the first token. Explaining the phenomenon,
\citet{sun2024massive} attribute the sink phenomenon to the massive activations which function as implicit attention biases. \citet{xiao2024efficient,gu2025attention} attribute the sink phenomenon to softmax normalization. \citet{gu2025attention} also studies the impact of data distribution, optimization, loss function and architecture on the formation of attention sinks. More recently, \citet{zhang2025attention} describe 
attention sinks as a mechanism intended to cluster the hidden representations of input tokens, and \citet{barbero2025llms} conclude that they are needed to prevent representational collapse.
\citet{guo2024active} study the formation of attention sinks using a simple Markovian data model and identify that heads may attend to sinks for some but not all input domains.

Substantial line of research has employed attention sinks to achieve practical benefits, particularly improvements in
key–value cache compression~\citep{guo-etal-2024-attention}, long-context generation~\citep{xiao2024efficient,Yang_2025_ICCV}, quantization~\citep{son-etal-2024-prefixing}, and feature map smoothness~\citep{darcet2024vision}.

\paragraph{Softmax alternatives and entropy collapse.}

Numerous alternatives to softmax attention have been proposed, mostly from the perspective of improved computational complexity and parallelizability, including sparsemax \citep{martins2016softmax}, linear attention \citep{pmlr-v119-katharopoulos20a, qin2022cosformer, wortsman2023replacing}, sigmoid attention \citep{ramapuram2025theory}, and others.

It has been shown that softmax leads to much lower attention entropy than its linear alternatives \citep{hong-lee-2025-variance}. However, the impact of this difference is not entirely clear: \citet{stabilizing-transformer-training} argue that entropy collapse leads to training instability and attempt to prevent it from happening, \citet{zhang2024hedgehog} state that low-entropy attention patterns are useful and design an attention operation to preserve them.

To the end of explaining entropy collapse, \citet{deng2024sparse} show that it manifests with high sequence length due to the log-scale increase in logits. Further, \citet{bao2024self} tie attention localization to the eigenspectrum of QK matrices. 

\paragraph{Sparse attention patterns.} Independently from the works on entropy collapse, it has been shown that softmax Transformers often implement sparse solutions, driving a lot of progress in Transformer interpretability \citep{olsson2022context, lindsey2025biology}. \citet{singh2024needs, zucchet2025emergence, yuksel2025incremental, varre2025learning} study the conditions required for the emergence of sparse attention patterns. 
\citet{nanda2023factfinding} show that LLMs exhibit subcircuits for factual recall, and
\citet{zucchet2025language} analyze the evolution dynamics of attention for this task.

\paragraph{Polarization and population dynamics.} From a technical perspective, the dynamical system analysis of our value-softmax model shares significant structural similarities with replicator dynamics in evolutionary game theory~\citep{taylor1978evolutionary, smith1982evolution}; see \citet{sorin2020replicator} for a survey and opinion polarization on social networks~\citep{proskurnikov2015opinion}. In these frameworks, the evolution of a state $x \in \mathbb{R}^{p}$ is generally governed by dynamics of the form:\begin{align*}\dot{x}_i \propto ( f(x_i) - \bar{f}(x) ),\end{align*}where $\bar{f}(x)$ represents the population (weighted-)aggregate. Because the rate of change is driven by the deviation from this aggregate, parameters tend to diverge from the mean. This mechanism underpins phenomena such as \emph{survival of the fittest} in biological contexts and the hardening of polarized opinions in social dynamics.

\section{Problem Setup}

\paragraph{Notations.} For $n \in \mathbb{N}$, let $[n] := \{0,1,\ldots,n{-}1\}$. Bold letters denote quantities of interest which includes time-dependent variables and their limiting values. For a vector $x \in \mathbb{R}^d$, $x_i$ denotes its $i^{th}$ coordinate. For a matrix $X \in \mathbb{R}^{m \times n}$, $X_i$ denotes its $i^{th}$ column. The notation $\|\cdot\|$ denotes the Euclidean norm for vectors and the Frobenius norm for matrices. For a vector $a \in \mathbb{R}^d$, $\mathrm{diag}(a)$ denotes the diagonal matrix with entries given by $a$. The all-ones vector is denoted by $\one$ and its dimension is inferred from the context. For vectors $u, v \in \mathbb{R}^d$, $u \odot v$ denotes the Hadamard (elementwise) product. 

The softmax function $\mathrm{\sf} : \mathbb{R}^p \to \mathbb{R}^p$ is defined by
\begin{align*}
  \sf(a)_i = \frac{\exp(a_i)}{\sum_{j=1}^{p} \exp(a_j)}.  
\end{align*}
We first present a single self-attention block and then simplify it to a model we call the value–softmax model.

\paragraph{Self attention.} 
Let $X \in \R^{T \times d}$ be a sequence of input token embeddings and let $Q,K,V$ denote the query, key, and value matrices of the attention block. The updated representation after the attention block is given by 
\begin{align*}
  X_{+} = \sf\left( X QK^{\top} X^{\top} \right) X V^{\top}. 
\end{align*}
Let $z$ denote the embedding of the last input token. Its updated embedding is 
\begin{align*}
  z_{+}  = V X^{\top} \sf\left( X K Q^{\top} z\right) = \bar{V}^{\top} \sf\left( \bar{a} \right),
\end{align*}
where we define $\bar{V} = XV^{\top} \in \R^{T \times d}$ and $\bar{a} = X K Q^{\top} z \in \R^{T}$. 
Such an update has the following structure: a trainable matrix $\bar{V}$ multiplied by softmax of a vector $\bar{a}$. The embedding 
\begin{align*}
z_{+} = \sum_{i} v_i \sigma(a)_i 
\end{align*}
can be represented in multiple ways. For example, the scores $\sigma_{i}'s$ may be uniform, in which case the vectors  $v_i'$ average to $z_{+}$; alternatively, the scores may be sparse(e.g., one-hot) in which case a single vector $v_i$ is selected to be $z_{+}$. 

We would like to understand which of these regimes the optimization procedure tends to favor. Note that in the above derivation we ignore the contribution of the value matrix $V$ to the updates of the other tokens and assume independence of the $z_{+}$. 
While these assumptions do not hold in general, they motivate a simplified model, defined next, that provides insight into the behavior induced by the softmax parameterization.

\paragraph{A value-softmax model.}
Let $\VV \in \R^{p \times p}$ and $\aa \in \R^{p}$ denote the trainable value matrix and score vector, respectively.  Let $\LL{\VV, \aa}$ be a loss function that depends only on the product  $\bbeta = \VV \sf(\aa)$:
\begin{align}
  \LL{\VV, \aa} = \bloss\left( \VV \sf(\aa) \right) = \bloss\left( \bbeta \right).
\end{align} 
The loss does not depend on $\aa$ and $\VV$ separately but only through the combination $\VV \sf(\aa)$. This model is a stylized simplification of the self-attention block.
Re-parameterized models are a popular setting to gain insight into the training dynamics of non-convex optimization problem~\citep{lireparam2022}.
For convenience, we denote $\sos = \sf(\aa)$.  

\paragraph{Gradient flow.}
We study the dynamics of the continuous time limit of gradient descent, the gradient flow. Although this analysis does not capture discrete-time effects, adaptive step sizes, or stochasticity, it nevertheless captures the principal behavior of first-order optimization and provides useful insight into the training process. 

The gradient flow for the parameters $\VV$ and $\aa$ is given by the differential equations 
\begin{align} \label{eq:gf-L}
  \der{\VV} &= - \nabla_{\VV} \LL{\VV, \aa} \der{t},  \,\,
  \der{\aa} = - \nabla_{\aa} \LL{\VV, \aa} \der{t}.
\end{align}
We exploit the  reparameterization $\LL{\VV,\aa} = \bloss\left( \VV \sf(\aa)\right)$ to obtain explicit expressions for the gradients. Using the chain rule on $\bbeta = \VV \sf\left( \aa \right)$, the gradient flow can be written as 
\begin{align}
  \der{\VV}  &= - \nabla_{\bbeta} \bloss(\bbeta) \sos^{\top} \der{t}, \\
  \der{\aa}  &= - \left[ \diag{\sos} - \sos \sos^{\top} \right] \VV^{\top} \nabla_{\bbeta} \bloss(\bbeta).
\end{align}
The Jacobian of the softmax $\diag{\sos} - \sos \sos^{\top}$ plays a central role in shaping the dynamics of  $\aa$.

\section{Polarizing effect of softmax}
In this section, we present the main results for the value–softmax model. We begin with the logistic loss and then discuss regression losses as well as other nonlinearities and normalization functions.

\subsection{Logistic loss}
%



We consider a binary classification setting, mirroring standard transformer training where the final representation is optimized with cross-entropy (i.e., logistic loss in the binary case). 
To isolate the effect of the value–softmax parameterization, we consider a stylized teacher setting in which the desired decision rule is given by a fixed linear classifier $\bbeta_* \in \R^{p}$.
The attention block produces $\bbeta$, and the logistic loss encourages a large positive margin $\scal{\bbeta_*} {\bbeta}$, i.e., agreement with the classifier induced by $\bbeta_*$.

This setting leads to the following formulation of the loss as a function of $\VV,\aa$:
\begin{align} \label{eq:logistic-loss-V-a}
  \LL{\VV, \aa} = \log\left( 1 + \exp \scal{\bbeta_*} {\, - \VV \sf(\aa)} \, \right),
\end{align}
The gradient flow dynamics are given by 
\begin{align*}
  \der{\VV} &= - \gamma(\bbeta) \, \bbeta_* \sos^{\top} \der{t}, \\
  \der{\aa} &= - \gamma(\bbeta) \, \left[ \diag{\sos} - \sos \sos^{\top} \right] \VV^{\top} \bbeta_* \der{t}.
\end{align*}
where $\gamma(\bbeta) = \left(1 + \exp\scal{\bbeta_*}{\bbeta}\right)^{-1}$.
Letting  $\uu = \VV^{\top} \bbeta_*$, the dynamics reduced to a couples system over two vectors $\aa, \uu \in \mathbb{R}^p$ 
\begin{align} \label{eq:logistic-loss-u-a}
  \mdot{\uu} &= - \gamma(\bbeta) \, \nor{\bbeta_*}^2 \sos, \\
  \mdot{\aa} &= - \gamma(\bbeta) \, \left[ \diag{\sos} - \sos \sos^{\top} \right] \uu.
\end{align}

\paragraph{Intuition from replicator dynamics.}
Inspecting the gradient-flow equations, the evolution of the scores is driven  by the centered signal
\begin{align*}
\mdot{\aa} \sim \left( \uu  -  \scal{\uu}{\sos} \one  \right) , \quad \twodot{\uu} \sim  \left( \uu  -  \scal{\uu}{\sos} \one  \right),
\end{align*}
where we omit scaling factors and normalization terms for clarity.  

Since $\sos=\sf(\aa)$ lies on the simplex, $\langle \uu,\sos\rangle$ is the $\sos$-weighted average of the coordinates of $\uu$. Thus, the term $\uu - \langle \uu,\sos\rangle \one$ measures each coordinate's deviation from the average.
This is reminiscent of replicator dynamics~\citep{sorin2020replicator}: coordinates with above-average ``fitness'' (larger $\uui{i}$) are amplified relative to those below average, polarizing the mass towards the best-performing coordinates. Motivated by this analogy, we next analyze how softmax normalization induces such polarization and sparsity in $\sos$.


\paragraph{Repulsion between the coordinates.}
To formalize the polarizing effect, we first show that the dynamics induce repulsion between the coordinates of the attention scores $\sos$ and the value projections $\uu$. We begin by stating an assumption on the initialization. 
\begin{assumption}\label{ass:init}
  The parameters $\VV, \aa$ are initialized such that $\aa(0) = 0$ and  $\uui{0}(0) > \uui{1}(0) > \ldots > \uui{p-1}(0)$ for the projection $ \uu = \VV^{\top} \bbeta_*$.
\end{assumption}
This assumption is mild. The condition $\aa(0)=0$ implies $\sos(0)$ is uniform, which is consistent with standard practice. Without loss of generality, we assume the coordinates of $\uu(0)$ are ordered by re-indexing. The strict ordering holds almost surely under generic random initialization and can be relaxed if needed. We now show that gradient flow preserves this ordering and induces repulsion between coordinates.
\begin{restatable}{theorem}{NoCrossingLogistic} \label{thm:no-crossing-logitstic}
 Consider gradient flow on the loss \eqref{eq:logistic-loss-V-a} under the initialization in \cref{ass:init}. Then, for all $t > 0$,
  \begin{enumerate}[label=(\alph*)]
    \item \textbf{Order preservation.} The order of both the projections $\uui{i}(t)$ and attention scores $\sosi{i}(t)$ is preserved
  \begin{align*}
    \uui{0}(t) &> \uui{1}(t) > \ldots > \uui{p-1}(t), \\
    \sosi{0}(t) &> \sosi{1}(t) > \ldots > \sosi{p-1}(t).
  \end{align*}
  \item \textbf{Repulsion.} Pairwise gaps between value projections grow over time: for any $i \neq j$ and  any  $t_+ > t \geq 0$,
  \begin{align*}
     \left( \uui{i}(t_+) - \uui{j}(t_+) \right) > \left( \uui{i}(t) - \uui{j}(t) \right).
  \end{align*}
  \end{enumerate}
\end{restatable}
\Cref{thm:no-crossing-logitstic} first establishes a \emph{no-crossing} property: the coordinate ordering is preserved along the trajectory. To prove this, we introduce the Lyapunov potential $\Phi_{ij} = - \left( e^{-\aai{i}} - e^{-\aai{j}}  \right) \left(  \uui{i} - \uui{j} \right)$ for $i\neq j \in [p]$. The derivative of potential is positive, $\der{\Phi_{ij}} > 0 $, hence the relative ordering between the scores and projections is preserved. In particular, $\Phi_{ij}(t) > 0$, for all $t > 0$, hence the co-ordinates never cross and the initial order remains conserved.
Second, the gradient flow exhibits repulsion among the coordinates of $\uu$: for $i < j$, we have $\der{ (\uui{i} - \uui{j}) } > 0$, driven by the preserved ordering of the attention scores. 
Together, these properties highlight the polarizing nature of the gradient field in the value–softmax model.

\begin{figure*}[t]
    \centering
    \includegraphics[width=\linewidth]{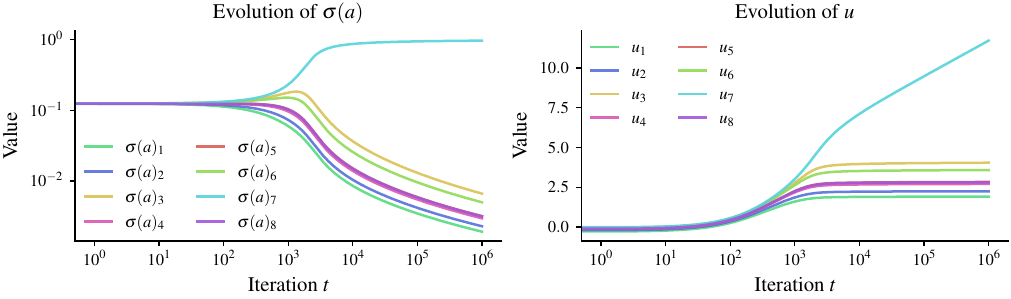}
    \caption{Experimental verification of Theorem \ref{thm:sparsity-logistic} using logistic loss. The plot shows the evolution of attention scores $\sigma(a)$ and value projections $u$ over time. As predicted, the attention scores converge to a one-hot vector (blue line goes to 1, others to 0) and the value projections $u$ diverge, preserving their order.}
    \label{fig:logistic_evolution}
\end{figure*}

\paragraph{Sparsity of the attention scores.}
Having established the polarizing nature of the dynamics, we now quantify its strength and characterize the resulting limiting behavior.
\begin{restatable}{theorem}{SparsityLogistic}
\label{thm:sparsity-logistic}
Consider gradient flow for the loss \eqref{eq:logistic-loss-V-a} under the initialization in \cref{ass:init}. Let $ \delta = \min_{i < j} \left( \uui{i}(0) - \uui{j}(0) \right) $.  Then we have 
\begin{enumerate}[label=(\alph*)]
    \item \textbf{Vanishing loss.}  
    $ \lim_{t \to \infty} L(\bbeta(t)) = 0 $.
    \item \textbf{Decay of non-maximal scores. } For any  $j \neq 0$ and any $t > 0$,
    \[ \frac{\sosi{j}(t)}{\sosi{0}(t)} \leq \frac{1}{1 + \nicefrac{\delta}{p} \int\limits_{0}^{t} \gamma(\bbeta(s)) ds }. \] 
    \item \textbf{One-hot limit.}
    The polarization coefficient $ \int_{0}^{t} \gamma(\bbeta(s)) ds = \Theta(\log t)$, and consequently the attention scores converge to a one-hot vector, i.e,
    \[ \lim_{t \to \infty} \sosi{0}(t) = 1, \quad \text{and} \quad \lim_{t \to \infty} \sosi{j}(t) = 0, \quad \text{for } j \neq 0. \]
\end{enumerate}
\end{restatable}
To prove the theorem, we track the ratio $\sosi{j}(t)/\sosi{0}(t)$ and show that it decays at a rate controlled by the integral $\int_{0}^{t}\gamma(\bbeta(s))\,ds$. 
We then prove that this integral diverges as $t\to\infty$, implying that $\sos(t)$ converges to a one-hot vector.
The softmax parameterization implicitly biases the dynamics toward low-entropy (peaked) attention.
Finally, note that,
\begin{align*}
  \bbeta(t) = \sum_j \sosi{j}(t) \vvi{j}(t).
\end{align*}
Thus, in the limit, $\bbeta(t)$ becomes supported on a single column of $\VV(t)$: among the many possible decompositions of $\bbeta$ as a convex combination of $\{\vvi{j}\}$, gradient flow selects an extremal (one-hot) representation. This provides a theoretical explanation for the empirically observed tendency of softmax attention to become sparse.

\paragraph{Comparison with the dynamics of linear networks.}
When $\sigma$ is the identity, the model reduces to a linear network, whose gradient-flow dynamics have been extensively studied in connection with implicit regularization of gradient-based methods. In particular, it is known that the solution converges to a rank-one factorization of $\VV $(see, e.g., \cite{ji2018gradient}). 
In our setting, a similar phenomenon occurs: under the softmax parameterization, $\VV(t)$ also becomes approximately rank one (see Appendix~\ref{lem:log-non-maximal-rates}). The crucial difference, however, is that the attention scores $\sos(t)$ converge to a one-hot vector. This behavior does not arise in linear networks, and it forces the dominant right singular direction of $\VV$ to align with a coordinate axis (i.e., to be one-hot) rather than a generic direction.

The polarization we observe is not caused by the softmax nonlinearity in isolation, but by the \emph{joint} training of the scores $\aa$ and the value matrix $\VV$. Their co-adaptation drives the dynamics toward low-entropy (peaked) attention patterns.

Finally, it is important not to conflate this effect with sparse regression: the sparsity here is in the attention scores $\sos$, not in the entries of the value matrix $\VV$ or final predictor $\bbeta$.

\subsection{Extensions of the result}
\label{sec:extensions}
We now explore how far our analysis extends beyond the logistic-loss setting. In particular, we vary the loss and the attention map to identify when polarization persists and when it fails.
%
\paragraph{Regression.}
We now consider a regression setting in which the target vector $\bbeta_*$ is to be fitted. The loss is 
\begin{align} \label{eq:regression-loss-V-a}
  \LL{\VV, \aa} = \frac{1}{2} \nor{ \bbeta_* - \VV \sf(\aa) }^2.
\end{align}
The associated gradient flow dynamics are
\begin{align*}
  \der{\VV} &=  \left(\bbeta_* - \VV \sos \right) \sos^{\top} \der{t}, \\
  \der{\aa} &=  \left[ \diag{\sos} - \sos \sos^{\top} \right] \VV^{\top} \left(  \bbeta_* - \VV \sos \right) \der{t}.
\end{align*}
We make the following initialization assumption.
\begin{assumption} \label{assm:regression-init}
  The parameters $\VV, \aa$ are initialized such that $\VV(0) = 0$ and the attention scores satisfy $ \sosi{0}(0) > \sosi{1}(0) > \ldots > \sosi{p-1}(0) $.
\end{assumption}
Initializing the linear predictor to zero is standard in analyses of gradient flow for linear regression and simplifies the resulting dynamics. We assume the attention scores are strictly ordered at initialization to avoid ties and break symmetry. Without loss of generality, we label the coordinates so that the ordering is decreasing. Under this initialization, $\VV(t)$ remains rank one, which allows us to reduce the dynamics to a coupled system similar to \cref{eq:logistic-loss-u-a}, as stated next.
\begin{restatable}{lemma}{Regression}
Consider gradient flow for the loss \eqref{eq:regression-loss-V-a} under \cref{assm:regression-init}. Then:
\begin{enumerate}[label=(\alph*)]
    \item \textbf{Rank-one structure.} 
 There exists $\uu(t)\in\R^p$ such that $\VV = \frac{\bbeta_*}{\|\bbeta_*\|^2} \uu^{\top} $.
    \item \textbf{Reduced dynamics.} The pair $\aa, \uu $ satisfies  
\begin{align} \label{eq:regression-loss-u-a}
  \mdot{\uu} &= - \gamma(\bbeta) \, \nor{\bbeta_*}^2 \sos, \\
  \mdot{\aa} &= - \gamma(\bbeta) \, \left[ \diag{\sos} - \sos \sos^{\top} \right] \uu, 
\end{align}
where $ \gamma(\bbeta) =  1 - \frac{\scal{\uu}{\sos}}{\nor{\bbeta_*}^2}$.
\end{enumerate}
\end{restatable}
Using the same techniques as in the logistic-loss case, we can derive an analogue of \cref{thm:no-crossing-logitstic} for the regression dynamics.
A key difference is that the cumulative polarization strength is finite:  $\int_{0}^{\infty} \gamma(\bbeta(s)) \der{s} = \Theta(1)$ (see Appendix for the result).
Consequently, polarization is only partial and $\sos(t)$ does not converge to a one-hot vector in general. More broadly, sparsification is controlled by the integral of the  gradient magnitude: the slower the gradient decays, the stronger the induced sparsity in the attention scores.
This suggests that ill-conditioned regression problems (e.g., $\| \bbeta_* - X \VV \sf(\aa) \|^2$ for a matrix  $X$ ) may exhibit stronger sparsification, since ill-conditioning slows down convergence. We empirically verify this relationship in Figures \ref{fig:main-cond-numbers} and \ref{fig:condition_numbers}.


\begin{figure}[b]
  \centering
  \includegraphics[width=\columnwidth]{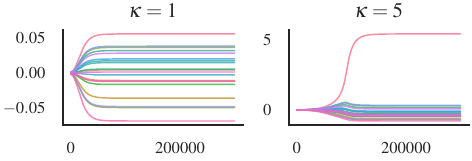}
  \caption{Evolution of attention logits in a regression problem with condition number $\kappa = 1$ and 5.}
  \label{fig:main-cond-numbers}
\end{figure}

\paragraph{ KL divergence.}
A natural alternative is the KL divergence between a target distribution $p_*$ and a predicted distribution $\bbeta$, e.g., $\ell(\bbeta) = -\langle p_*, \log \bbeta\rangle$, where the logarithm is applied elementwise. The resulting gradient-flow dynamics are analogous to the regression case, and we again observe polarization, albeit with diminishing strength compared to the case of classification (Figure \ref{fig:kl}).

\paragraph{Other non-linearities and normalization functions.} 
We first consider replacing softmax with elementwise nonlinearities such as the sigmoid activation  $\sigma_{sig}(a)_i = \nicefrac{1}{1 + \exp(-a_i)}$ or the Relu. For both the regression and logistic losses, we do not observe the same polarization behavior under sigmoid (and similarly under ReLU), see Figure~\ref{fig:toy-many-plots} for reference. Intuitively, the key mechanism behind polarization in our analysis is the mean-centering term that appears in the softmax Jacobian; elementwise nonlinearities such as sigmoid or ReLU do not induce an analogous ``fitness minus average fitness'' interaction across coordinates.

We next consider alternative \emph{normalization} maps, which retain this interaction structure. Let $f$ be a monotonically increasing function and define the normalization $\sf_{f}: \R^{p} \to \R^{p}$ by
\begin{align*}
  \sf_{f}(a)_i = \frac{f(a_i)}{\sum_{j=1}^{p} f(a_j)}.
\end{align*}
Writing $\sos_f \coloneqq \sf_f(\aa)$, the gradient flow dynamics for the logistic loss take the form
\begin{align*}
  \der{\uu} &=  \gamma(\bbeta) \, \|\bbeta_*\|^2   \sos_{f} \, \der{t}, \\
  \der{\aa} &=  \gamma(\bbeta) \, \frac{f'(\aa)}{\sum_{j=1}^{p} f(\aa_j)} \odot ( \uu - \scal{\uu}{\sos_f} \one ) \, \der{t},
\end{align*}
This retains the replicator-type structure, and an analogue of \cref{thm:no-crossing-logitstic} can be established. 
The polarization strength now depends on $f$ and $f'$. For instance, with $f(x)=x$ (which does not enforce positivity), the polarization is typically insufficient to yield one-hot scores, where as with $f(x)=x^2$ the dynamics can drive $\sos_f$ toward a one-hot limit, see Figure~\ref{fig:toy-many-plots} provides empirical evidence for these behaviors.

\paragraph{Limitations.}
A key limitation of our analysis is that, for the losses considered above, the gradient $\nabla \ell(\bbeta)$ is aligned with the target direction $\bbeta_*$ (in the regression case, this relies on \cref{assm:regression-init}). 
This alignment substantially simplifies the dynamics and enables the reduction to the two-vector system in $\aa$ and $\uu$. 
In practice, such an alignment need not hold, and the present analysis may therefore not apply directly. Nevertheless, our experiments indicate that the polarization effect persists even without exact alignment; see \cref{fig:main-cond-numbers} as alignment breaks in the large condition number case.
 
\section{Implications of softmax-induced polarization for attention sparsity and token influence}
\label{sec:implications}
We now connect our theoretical polarization result to qualitative behaviors observed in attention mechanisms. In particular, we discuss how softmax-induced sparsity can lead to attention sinks (and associated large activations) and to an imbalance in how individual tokens influence the model’s predictions. The code is available at \url{https://github.com/tml-epfl/softmax}

\subsection{Attention sinks and massive activations}
\paragraph{Attention sinks.} 
Attention sinks~\citep{xiao2024efficient} are attention patterns in which the maximal attention is over a specific token (often special tokens such as BOS or delimiter/punctuation tokens).
Attention sinks are suggested to be formed when the heads need to be turned off or when they have to create a explicit bias due to absence of biases in self-attention architecture~\citep{sun2024massive,zhang2025attention,gu2025when}. 
However, to create a bias, there is no reason to focus all attention on a single token; a distributed or uniform pattern over a set of tokens could also create such a bias.
%
%

\begin{table}[t]
\centering
\small
\setlength{\tabcolsep}{6pt}
\renewcommand{\arraystretch}{1.35}

\resizebox{\linewidth}{!}{%
\begin{tabular}{l | l | l}
\toprule
Name & $\mathrm{sim}(\phi(\mathbf q_i), \phi(\mathbf k_j))$ & Normalization \\
\midrule
softmax &
$\displaystyle \exp\!\left(\frac{\mathbf q_i \mathbf k_j^\top}{\sqrt{d_h}}\right)$ &
$\displaystyle \sum_{j'=1}^{m}\exp\!\left(\frac{\mathbf q_i \mathbf k_{j'}^\top}{\sqrt{d_h}}\right)$ \\

sigmoid &
$\displaystyle \sigma\!\left(\frac{\mathbf q_i \mathbf k_j^\top}{\sqrt{d_h}}\right)$ &
$\displaystyle \sum_{j'=1}^{m}\sigma\!\left(\frac{\mathbf q_i \mathbf k_{j'}^\top}{\sqrt{d_h}}\right)$ \\

elu &
$\displaystyle \mathrm{elu}\!\left(\frac{\mathbf q_i \mathbf k_j^\top}{\sqrt{d_h}}\right)+1$ &
$\displaystyle \sum_{j'=1}^{m}\Big(\mathrm{elu}\!\left(\frac{\mathbf q_i \mathbf k_{j'}^\top}{\sqrt{d_h}}\right)+1\Big)$ \\

elu $\cdot$ elu &
$\displaystyle \frac{\big(\mathrm{elu}(\mathbf q_i)+1\big)\big(\mathrm{elu}(\mathbf k_j)+1\big)^\top}{\sqrt{d_h}}$ &
$\displaystyle \sum_{j'=1}^{m}\frac{\big(\mathrm{elu}(\mathbf q_i)+1\big)\big(\mathrm{elu}(\mathbf k_{j'})+1\big)^\top}{\sqrt{d_h}}$ \\

mlp &
$\displaystyle \frac{\mathrm{mlp}(\mathbf q_i)\,\mathrm{mlp}(\mathbf k_j)^\top}{\sqrt{d_h}}$ &
$\displaystyle \max\!\left(\left|\sum_{j'=1}^{m}\frac{\mathrm{mlp}(\mathbf q_i)\,\mathrm{mlp}(\mathbf k_{j'})^\top}{\sqrt{d_h}}\right|,\,1\right)$ \\

linear &
$\displaystyle \frac{\mathbf q_i \mathbf k_j^\top}{\sqrt{d_h}}$ &
$\displaystyle \max\!\left(\left|\sum_{j'=1}^{m}\frac{\mathbf q_i \mathbf k_{j'}^\top}{\sqrt{d_h}}\right|,\,1\right)$ \\

\bottomrule
\end{tabular}%
}
\vspace{2pt}
\caption{Attention functions we test in Section \ref{sec:implications}. When computing the normalization, $m =$ sequence length for non-causal models and $m = i$ for causal models. For all functions but softmax, we test their variants both with and without normalization.}
\label{table:attention-types}
\end{table}

To shed light on why optimization may nevertheless favor sinks, consider the value--softmax model in which the attention output is
$\bbeta = \sf(\bA)\VV$, where $\bA\in\mathbb{R}^{T\times T}$ is the attention scores matrix (softmax applied row-wise) and $\VV\in\mathbb{R}^{T\times d}$ are value vectors.
Suppose we optimize $(\VV,\bA)$ so that each position $t$ produces the same target vector $\bbeta_*\in\mathbb{R}^d$, e.g.,
\begin{align*}
  \LL{\VV, \bA} = \ell(\bbeta) = \frac{1}{T} \sum_{i=1}^{T} \log\left( 1 + \exp\scal{\bbeta_*}{- \bbeta[t]} \right).
\end{align*}
Extending \cref{thm:sparsity-logistic}, the row-wise softmax weights converge to a one-hot vector supported on the initially maximal coordinate of $\uu(0)=\VV^\top\bbeta_*$.
Under Assumption~\ref{ass:init} this coordinate is $0$ for all rows, so $\sf(\bA[t,:])\to e_0$ for every $t$, yielding an attention sink at token $0$. This result suggests that attention sinks \emph{can} emerge as a consequence of the optimization dynamics induced by the softmax parameterization.

\paragraph{Massive activations.} 
\citet{sun2024massive} has reported that attention sinks are also coupled with massive activations in some feature dimension, i.e., a regime where a small fraction of activations takes values that are significantly larger than the rest.
A reformulation of the value--softmax viewpoint provides intuition for how this can arise from the same polarizing dynamics.
Recall that the forward pass for some intermediate residual state $R \in T \times T$, and the residual state of the last token $r$ is given by 
\begin{align*}
  R_{+} &=   \sf\left( R  QK^{\top} R^{\top} \right) R V, \\
  r_{+} &=  V^{\top} R^{\top} \sf( RKQ^{\top}z). 
\end{align*}
Note that here the product is between the residual state matrix and the softmax outputs, similar to the value-softmax model with an exception that the softmax inputs are now not free. One can express this idea with a formulation
$\LL{\RR, \aa}= \bloss\left( \RR \sf(\RR \aa) \right)$, where $\bloss$ is a logistic loss function, and see that this model qualitatively follows the behavior of the value-softmax model. 
In Figure \ref{fig:massive_activations}, we can see the polarization in softmax scores and also the polarization in the embeddings of the residual state for this model. This shows that the attention sinks and massive activations  can arise from underlying the polarizing dynamics. 

\paragraph{Induction heads acting as attention sinks.}

\begin{figure}[t]
  \centering
  \includegraphics[width=\columnwidth]{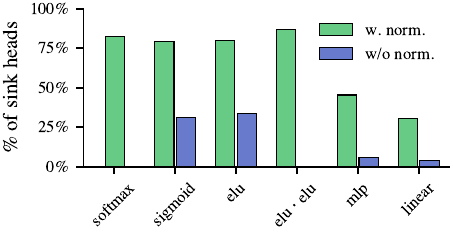}
  \caption{Proportion of sink heads at the 2nd layer of Transformer trained for the induction task.}
  \label{fig:main-exp-induction}
\end{figure}

To test whether the results proven for the value-softmax model hold for more general Transformer setups, we run additional experiments in more realistic transformer model (details in Appendix \ref{app:experimental-details}). 
First, we create a simple setup where attention sinks emerge. We train 2-layer Transformers with a variable number of attention heads on an induction task: predict a bi-gram learned in-context. Our task is a simple variant of Bigram-Backcopy task of  ~\citep{guo2024active,bietti2024birth}. Each input object is a sequence of tokens sampled uniformly at random, 
followed by a bigram, that can be predicted by a construction known as \emph{induction heads} \citep{olsson2022context}. Example input is shown below:
\begin{figure}[h]
\vspace{-4pt}
\centering
{
\textcolor{black!70}{\textbf{BOS}}\;
\textcolor{black!70}{4}\;
\textcolor{black!70}{2}\;
\textcolor{black!70}{6}\;
\textcolor{blue!70!black}{12}\;
\textcolor{green!60!black}{8}\;
\textcolor{black!70}{11}\;
\textcolor{blue!70!black}{12$^\ast$}\;
\textcolor{green!60!black}{8}
}
\vspace{-15pt}
\end{figure}

At the positions where tokens are sampled at random,  an induction heads need not appear and they can act as a fixed attention bias similar to the 
%
the emergence of attention sinks as described in Section \ref{sec:implications}. 
On this task, we evaluate the formation of attention sink and its dependence on various choices of functions listed in Table~\ref{table:attention-types} as an alternative to compute attention. 
We define an attention head as a \emph{sink head} if, at positions where tokens are sampled uniformly at random, it assigns more than 90\% of the total attention weight to the first token. We report the average proportion of such heads in Figure \ref{fig:main-exp-induction}.
In line with the results in Section \ref{sec:extensions}, the normalized positive attention variants are associated with consistently higher proportion of sink heads than their unnormalized counterparts or non-positivite variants (mlp, linear).

\begin{figure}[t]
  \centering
  \includegraphics[width=\columnwidth]{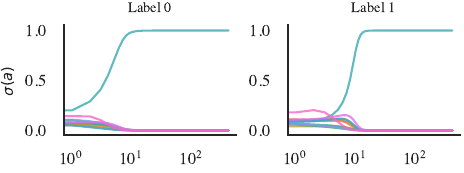}
  \caption{Evolution of the attention scores for two differently labeled samples in the classification experiment.}
  \label{fig:exp-classifcation-attentions}
\end{figure}

\begin{figure}[b]
  \centering
  \includegraphics[width=\columnwidth]{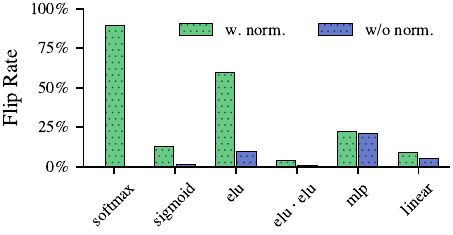}
  \caption{Success rates of adversarial token 1-token flips in 1-layer 1-head Transformer trained for classification.}
  \label{fig:main-exp-classification}
\end{figure}


\paragraph{Pretrained LLMs.}
Additionally, we verify whether the observations from toy setups hold for real-world large language models. We compare softmax and sigmoid (unnormalized) 7-billion-parameter LLMs matched by hyperparameters and compute, released by \citet{ramapuramtheory}. We run a forward pass of both models on random samples from the Pile dataset \cite{gao2020pile} and measure the sparsity score of each head as the average proportion of total attention weight assigned to the max-logit token. Figure \ref{fig:main-llm-pretrained} shows the distribution of this metric across all heads in both models, confirming that softmax leads to significantly sparser attention outputs and a higher likelihood of sink formation. This result aligns with \citet{gu2025when}, who reported similar behaviors in models with 1B parameters.


\subsection{Imbalance in token influence}


A second major implication of our results is the emergence of tokens that exert a disproportionate influence on the model's output. As predicted by Theorem~\ref{thm:sparsity-logistic}, when attention weights converge to a one-hot vector, the output effectively ignores all context except for the single 'max-logit' token. Consequently, the model's predictions become highly sensitive to perturbations of that specific token. 





We construct a simple classification dataset in the following way:
Consider a vocabulary $V$ and partition $V$ into $k$ sets $V^{(1)},\ldots,V^{(K)}$ of equal size. An input and its corresponding label are  
$x_i = v^{(k)}_{(i,1)},\dots,v^{(k)}_{(i,n)}, y_{i}=k$, where $x_i$ is a sequence of $n$ tokens sampled from $V^{(K)}$ and labeled $k$. Note that as the sets are disjoint any one of the input tokens unambiguously predicts the label, thus any form of attention map allows to solve the classification problem.

Due to the low-entropy bias of softmax, however, we expect that the attention weights will concentrate on few tokens, effectively ignoring the others. We test this by examining the behavior of the classifier in adversarial out-of-distribution setting where we are allowed to change one token in the input with the goal of flipping the classifier's prediction. The success rates of this experiment are presented in Figure \ref{fig:main-exp-classification}. As expected, the softmax model's predictions can be easily reversed, however the effect gradually vanishes as more than 1 head and layer are added to the model, as different heads, although sparse, concentrate attention on distinct tokens.



\begin{figure}[t]
  \centering
  \includegraphics[width=\columnwidth]{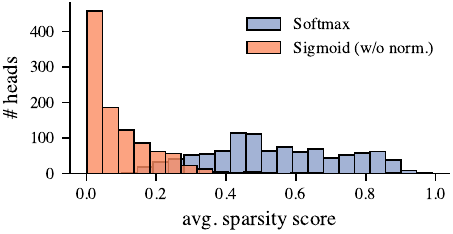}
  \caption{Distribution of attention head sparsity scores for pretrained softmax/sigmoid LLMs.}
  \label{fig:main-llm-pretrained}
  \vspace{-2pt}
\end{figure}

\section{Conclusion}

This work isolates an implicit optimization bias in the value–softmax parameterization $\VV\sf(\aa)$: under gradient flow, jointly training the scores and values  polarizes the attention distribution toward low-entropy solutions, even when many higher-entropy decompositions represent the same predictor. In a stylized classification setting, we prove that $\sf(\aa)$ converges to a one-hot vector, selecting an extremal representation. We extend the analysis to regression and related objectives, clarifying how convergence speed and activation choices modulate polarization. 
Finally, we connect this dynamics to transformer behavior and identify its connection with attention sinks and associated pathologies. 
We support the theory with experiments showing how attention type impacts sparsity in models trained for induction and classification, as well as for pretrained LLMs.

\bibliographystyle{abbrvnat}
\bibliography{references}

\newpage
\appendix
\onecolumn

\section{Some Properties of the Value-softmax Model}

\paragraph{Notation.} For any function $f : \R^{m} \to  \R^{n}$, for $x \in \R^{m}$, denotes $\jacb{x}{f} \in \R^{n \times m}$ denotes Jacobian of the function.

\paragraph{Gradient Flow Dynamics.} We will consider the gradient flow dynamics of 
the loss function $\loss(\VV, \aa)$ given by
\begin{align} \label{eq:gradient-flow}
    \frac{d\VV}{dt} &= - \nabla_{\VV} \loss(\VV, \aa),  \quad 
    \frac{d\aa}{dt} = - \nabla_{\aa} \loss(\VV, \aa).
\end{align}
We will use the chain rule to compute the gradients of $\loss(\VV, \aa)$ as follows:
\begin{align*}
\jacb{\VV}{\loss(\VV, \aa)} &= \jacb{\bbeta}{\ell} \jacb{\VV}{\bbeta(\VV, \aa)}, \\ 
\jacb{\aa}{\loss(\VV, \aa)} &= \jacb{\bbeta}{\ell} \jacb{\aa}{\bbeta(\VV, \aa)}.
\end{align*}
Also, with an abuse of notation, we also denote the jacobian of $\bbeta$ with respect to $\sos$.  
\begin{align*}
    \jacb{\sos}{\loss(\VV, \aa)} = \jacb{\bbeta}{\loss} \jacb{\sos}{\bbeta(\VV, \aa)}, \\
\end{align*}
Using the properties of Jacobians in Property~\ref{prop:jacobians} and using the notation, 
\begin{align*}
 \rrr^{\top} = - \jacb{\bbeta}{\loss}, \text{ i.e., } \quad \rrr = - \nabla_{\bbeta} \loss. 
\end{align*}
We can write the Jacobians as follows, 
\begin{align*}
    \jacb{\VV}{\loss(\VV, \aa)} &=  \jacb{\bbeta}{\ell} \, \, \left[ \sos^{\top} \otimes \I_{d} \right] = - \rrr^{\top}  \left[ \sos^{\top} \otimes \I_{d} \right] = - \sos^{\top} \otimes \rrr^{\top}.   \\
    \jacb{\aa}{\loss(\VV, \aa)} &=  \; \jacb{\bbeta}{\ell} \, \left[ \VV \left( \diag{\sos} - \sos \sos^{\top} \right) \right] = - \rrr^{\top} \left[ \VV \left( \diag{\sos} - \sos \sos^{\top} \right) \right]. \\
\end{align*} 
With $\rrr = - \nabla_{\bbeta} \loss$,  the jacobians now can be written as gradients in the following way:
\begin{align} \label{eq:gradients-general-loss}
    \nabla_{\VV} \loss(\VV, \aa) &=  - \rrr \sos^{\top},  \\
    \nabla_{\aa} \loss(\VV, \aa) &=  - \left( \diag{\sos} - \sos \sos^{\top} \right) \VV^{\top} \rrr.
\end{align}
\section{Logistic Regression}

In this section, we will analyze the gradient flow dynamics of the value-softmax model over the logistic loss function. First, let us define the logistic loss function over the parameter $\bbeta \in \R^{p}$ as follows:
\begin{align} \label{eq:logistic_beta}
    \bloss(\bbeta) = \log \left( 1 + \exp\left( -  \scal{\bbeta_*}{\bbeta}  \right) \right).
\end{align}
This just changes the scale of the problem and does not affect the analysis. 
\paragraph{Gradient flow over logistic loss.} The gradient of the logistic loss function is given as follows:
\begin{align*}
    - \nabla_{\bbeta} \bloss(\bbeta) = \frac{1}{1 + \exp\left(  \scal{\bbeta_*}{\bbeta}  \right)} \bbeta_*.
\end{align*}
Note that it is always along the direction of $\bbeta_*$ and the scale depends on the parameter $\bbeta$.  Denote the scalar term in the front of $\bbeta_*$ as $\gamma(\bbeta)$, i.e.,
\begin{align*}
    \gamma(\bbeta) \coloneqq \frac{1}{1 + \exp\left(  \scal{\bbeta_*}{\bbeta}  \right)}.
\end{align*}
From now, we will drop the $\bbeta$ and denote it just by $\gamma$.  Using the gradient computation in Eq.~\eqref{eq:gradients-general-loss}, the gradient flow dynamics of $\VV$ and $\aa$ writes:
\begin{align*}
    \frac{d\VV}{dt} &=  \gamma \, \bbeta_* \sos^{\top},  \quad 
    \frac{d\aa}{dt} =  \gamma \, \left( \diag{\sos} - \sos \sos^{\top} \right) \VV^{\top} \bbeta_*.
\end{align*}
The evolution of $\sos$ can be written as:
\begin{align*}  
        \frac{d\sos}{dt} &=  \gamma \, \left( \diag{\sos} - \sos \sos^{\top} \right)^2 \VV^{\top} \bbeta_*.
\end{align*}
Tracking the evolution of the projection vector $\uu$ given by $\uu \coloneqq \VV^{\top} \bbeta_*$, we have,
\begin{align*}
    \frac{\der{\uu}}{\der t} &= \gamma \sos \nor{\bbeta_*}^2.
\end{align*}
This gives us the system of equations governing the gradient flow dynamics of the value-softmax model over the logistic loss function as follows:
\begin{align} 
  \mdot{\uu} &= - \gamma(\bbeta) \, \nor{\bbeta_*}^2 \sos, \\
  \mdot{\aa} &= - \gamma(\bbeta) \, \left[ \diag{\sos} - \sos \sos^{\top} \right] \uu.
\end{align}

\paragraph{No-crossing property.} Now that we have established the gradient flow dynamics of the value-softmax model over the logistic loss function, we will analyze some properties of the dynamics. We will now give a lemma which highlights the "no-crossing" property of the dynamics. The lemma proceeds by studing the differences between co-ordinates of the projection vector $\uu$ and $\aa$.



\NoCrossingLogistic*
\begin{proof}
    Let $     \cA = \sum_{i} \exp{\aai{i}} $, and note that $\sosi{i}{=} \, \nicefrac{\exp{\{\aai{i}\}}}{\cA}$. Using this, we can write the evolution of $\aai{i}$ as follows:
\begin{align*}
    \frac{d\aai{i}}{dt} &=  \gamma \frac{\exp\{\aai{i}\}}{\cA}  \left[ \uui{i} - \gamma \scal{\sos}{\uu} \right], \\
    \frac{d\{ {-} \exp\{{-\aai{i}} \}\}}{dt} &=   \frac{\gamma }{\cA}  \left[ \uui{i} - \gamma \scal{\sos}{\uu} \right].
\end{align*}
Taking the differences between the evolution of two co-ordinates $i$ and $j$, we have,
\begin{align*}
    \frac{d \{ {-}\left( \exp{-\aai{i}} - \exp{-\aai{j}} \right) \}}{dt} &=   \frac{\gamma }{\cA}  \left[ \uui{i} - \uui{j} \right], \\
    \frac{ d \{ \uui{i} - \uui{j} \} }{ d t} &=  \gamma \, \left( \frac{\exp{\{\aai{i}\}} - \exp{\{\aai{j}\}}}{\cA} \right) \, \| \bbeta_* \|^2. 
\end{align*}
Notice how the differences reinforce each other, to formulate this we will consider the following Lyapunov function. For $i,j \leq p$,
\begin{align*}
    \Phi_{ij}(t) = (\uui{i}(t) - \uui{j}(t)) \{ {-}\left( e^{-\aai{i}(t)} - e^{-\aai{j}(t)} \right) \}.
\end{align*}

Consider the evolution of the potential $\Phi_{ij}(t)$, note that, due to the  
\begin{align*}
    \frac{d}{dt} \left[ (\uui{i} - \uui{j}) \{ {-}\left( e^{-\aai{i}} - e^{-\aai{j}} \right) \}  \right] &= \frac{\gamma}{\cA} \left( (\uui{i} - \uui{j})^2 \right) + \frac{\gamma}{\cA} \| \bbeta_* \|^2 \{e^{\aai{i}} - e^{\aai{j}}\} \{ {-}\left( e^{-\aai{i}} - e^{-\aai{j}} \right) \},  \\
    \frac{d}{dt} \left[ (\uui{i} - \uui{j}) \{ {-}\left( e^{-\aai{i}} - e^{-\aai{j}} \right) \}  \right] &= \frac{\gamma}{\cA} \left( (\uui{i} - \uui{j})^2 \right)
     + \frac{\gamma}{\cA} \| \bbeta_* \|^2 \frac{ \left(e^{\aai{i}} - e^{\aai{j}} \right)^2 }{e^{\aai{i}} e^{\aai{j}} }.
\end{align*}

    Note that at $t=0$, we have $\aai{i}(0) = \aai{j}(0) = 0$, $\sosi{i}(0) = \sosi{j}(0) = {1}/{p}$ and $\uui{i}(0) > \uui{j}(0)$. So $\Phi_{ij}(0) = 0$ and $\Phi_{ij}(t) > 0$, for any time $t > 0$. The strict inequality comes from the gradient at $t{=}0$ is non-zero. 

    Note that for any $i,j$, we have, $\Phi_{ij} > 0$. Hence, there cannot exist a time $t$ such that $\uui{i}(t) = \uui{j}(t)$ or $\aai{i}(t) = \aai{j}(t)$. Hence the attention scores and the projections cannot cross each other. Hence the order  $\uui{0} > \uui{1} > \uui{2} > \ldots $ at initialization is preserved for all time and also the order $\aai{0} > \aai{1} > \aai{2} > \ldots $ is enforced at time $t > 0$ due to the positivity of $\Phi_{ij}(t)$.
   
The other implication of this no-crossing property is that the differences increase with time. 
For any $i < j$, we have,
\begin{align*}
\frac{ d \{ \uui{i} - \uui{j} \} }{ d t} &=  \gamma \, \left( \frac{\exp{\{\aai{i}\}} - \exp{\{\aai{j}\}}}{\cA} \right) \, \| \bbeta_* \|^2 > 0, \text{ for } t > 0. 
\end{align*}
Hence, $\uui{i} (t_+) - \uui{j} (t_+) > \uui{i} (t) - \uui{j} (t) $ for any $t_+ > t > 0$. A similar argument holds for $\aai{i}$ and $\sosi{i}$.
\end{proof}

\paragraph{Sparsity of the attention.} In the previous lemma, we established the no-crossing property of the dynamics and the repulsive forces. In this section, we will analyze the strength of this repulsive feild to establish the sparsity of the attention scores. To begin with, we will analyze the ration of the scores of two co-ordinates $i$ and $j$. Along with this, we will assume that $$\nor{\bbeta_*} = 1.$$

\SparsityLogistic*
\begin{proof}
 To prove this theorem, we will track the ratio of weights of the attention scores $\sos$'s with  to the highest attention score $\sosi{0}$ . Let us denote the ratio of the attention scores as follows: 
\begin{align*}
    \qqi{i}(t) = \frac{\sosi{i}(t)}{\sosi{0}(t)} = \exp{\left( \aai{i}(t) - \aai{0}(t) \right)}.    
\end{align*}
Using the $\qqi{i}$'s the softmax weights can be written as,
\begin{align*}
\sosi{i} &= \frac{\qqi{i}}{1 + \sum_{j = 1}^{d-1} \qqi{j}} = \frac{\qqi{i}}{\cQ},  \\
\end{align*}  
where
\begin{align*}
    \cQ &= 1 + \sum_{j = 1}^{d-1} \qqi{j}. 
\end{align*}

\paragraph{Evolution of the ratios.} To compute the evolution of $\qqi{i}$, we will first compute the following dynamics,

\begin{align*}
    d {e^{-\aai{i}}} &= - \frac{\gamma}{\cA} \left( \uui{i} - \scal{\uu}{\sos} \right) d{t}, \\
    d { \left( \uui{i} - \uui{j} \right) } &= \frac{\gamma}{\cA} \left( e^{\aai{i}} - e^{\aai{j}} \right) d{t}, \\
    d { \left( e^{-\aai{i}} - e^{-\aai{j}} \right) } &= - \frac{\gamma}{\cA} \left( \uui{i} - \uui{j} \right) d{t}. \\
    d { \, e^{\aai{i}}} &= \gamma \frac{\left(e^{\aai{i}}\right)^2}{\cA} \left( \uui{i} - \scal{\uu}{\sos} \right) d{t}.
\end{align*}
Lets compute, 
\begin{align*}
    d \, \qqi{i} =  d { \left( e^{ \aai{i}}  e^{-\aai{0}}\right) } &= e^{-\aai{0}} d {e^{\aai{i}}} + e^{\aai{i}} d {e^{-\aai{0}}}, \\     
    &= \gamma \frac{e^{-\aai{0}} \left(e^{\aai{i}}\right)^2}{\cA} \left( \uui{i} - \scal{\uu}{\sos} \right) d{t} - \gamma \frac{e^{\aai{i}}}{\cA} \left( \uui{0} - \scal{\uu}{\sos} \right) d{t}, \\
    &= \gamma \frac{e^{\aai{i}}}{\cA} \left[ e^{-\aai{0}} e^{\aai{i}} \left( \uui{i} - \scal{\uu}{\sos} \right) - \left( \uui{0} - \scal{\uu}{\sos} \right) \right] d{t}, \\
    &= \gamma \sosi{i} \left[ \,  \qqi{i} \left( \uui{i} - \scal{\uu}{\sos} \right) - \left( \uui{0} - \scal{\uu}{\sos} \right) \right] d{t}, \\
    &= \gamma \, \, \frac{\qqi{i}}{\cQ} \, \, \left[ \,  \qqi{i} \left( \uui{i} - \scal{\uu}{\sos} \right) - \left( \uui{0} - \scal{\uu}{\sos} \right) \right] d{t}.
\end{align*}

Using the above derivation, the dynamics of $\qqi{i}$ can be written as,
\begin{align*}
    \der{\qqi{i}} &= \gamma \frac{\qqi{i}}{\cQ} \left[ \qqi{i} \left( \uui{i} - \scal{\uu}{\sos} \right) - \left( \uui{0} - \scal{\uu}{\sos} \right) \right] \der{t}, \\
    &= \gamma \frac{\qqi{i}}{\cQ} \left[ - \qqi{i} \left( \uui{0} - \uui{i} \right) - \left(  1 - \qqi{i} \right)\left(  \uui{0} - \scal{\uu}{\sos} \right) \right] \der{t}.
\end{align*}

From the previous therorem all time $t \geq 0$, the following properties hold along the trajectory of the gradient flow, 
\begin{enumerate}[label=(\alph*)]
    \item For all $i = 1, 2, \ldots, p-1$, $\, \, 0 \leq \qqi{i}(t) \leq 1$  and therefore $ 1 \leq \cQ \leq p$ for all time $t \geq 0$.
    \item For all $i = 1, 2, \ldots, p-1$, $\uui{0}(t) - \uui{i}(t) \geq \delta $ for all time $t \geq 0$.
\end{enumerate}

Using these properties, we will have some asymptotic bounds on relevant quantities. First, the evolution of $\qqi{i}$ can be bounded as follows,
\begin{align*}
    \der{\qqi{i}} &\leq -  \frac{\gamma}{\cQ} \left[ \qqi{i}^2 + \qqi{i} (1 - \qqi{i}) ( 1 - \frac{1}{\cQ}) \right] \delta \der{t}, \\
    \der{\qqi{i}} &\leq - \frac{\gamma}{\cQ} \left[ \qqi{i} - \frac{\qqi{i}(1-\qqi{i})}{\cQ} \right] \delta \der{t} , \\
    &\leq -\frac{\gamma}{\cQ} \left[ \qqi{i} - \qqi{i}(1 - \qqi{i}) \right] \ delta \der{t}, \\
    &\leq - \frac{\gamma}{\cQ} \qqi{i}^2 \delta \der{t} \leq -  \frac{\gamma}{p} \qqi{i}^2 \delta \der{t}.
\end{align*} 
Using this, we can have the following bound on $\qqi{i}$,
\begin{align*}
    \int_{1}^{\qqi{i}(t)} \frac{1}{q^2} \der{q} &\leq - \frac{\delta}{p} \int_{0}^{t} \gamma(h) \der{h}, \\
    - \left( \frac{1}{\qqi{i}(t)} - 1 \right) &\leq - \frac{\delta}{p} \int_{0}^{t} \gamma(h) \der{h}, \\
    \qqi{i}(t) &\leq \frac{1}{1 + \frac{\delta}{p} \int_{0}^{t} \gamma(h) \der{h}}.
\end{align*}
Note that this bound is irrespective of $i$. The goal of the rest of the proof is to show that the integral $\int_{0}^{t} \gamma(h) \der{h}$ diverges as $t \to \infty$.
Recall that,
\begin{align*}
    \gamma = \frac{1}{1 + \exp\left(  \scal{\bbeta_*}{\bbeta}  \right)} = \frac{1}{1 + \exp\left(  \sum_{j} \sosi{j} \uui{j}  \right)}. 
\end{align*} 
In the above expression, $\uui{j} \leq \uui{0}$ for all $j$.
\begin{align*}
    \uui{j} \leq \uui{0} &\implies \sum_{j} \sosi{j} \uui{j} \leq \uui{0} \sum_{j} \sosi{j} = \uui{0}, \\
    &\implies 1 + \exp\left(  \sum_{j} \sosi{j} \uui{j}  \right) \leq 1 + \exp\left(  \uui{0}  \right), \\
    &\implies \gamma \geq \frac{1}{1 + \exp\left(  \uui{0}  \right)}.
\end{align*}
First, lets check evolution of $\uui{0}$, 
\begin{align*}
    d{\uui{0}} &=  \gamma  \sosi{0} dt =  \gamma \frac{1}{\cQ} d{t} \geq \frac{\gamma}{p} d{t} \geq \frac{1}{p(1 + \exp( \uui{0} ))} d{t}.
\end{align*}
From this evolution we have the following, 
\begin{align*}
    \int_{0}^{t} d h &\leq p \int_{\uui{0}(0)}^{\uui{0}(t)} (1 + \exp(z)) \der{z}, \\
    t &\leq p \left[ \uui{0}(t) - \uui{0}(0) + \exp(\uui{0}(t)) - \exp(\uui{0}(0)) \right], \\
    &\leq p \left[ \uui{0}(t) + \exp(\uui{0}(t)) + c_0 \right], \\
    &\leq p \left[ 2 \exp(\uui{0}(t)) + c_0 \right], \\
    &\implies \exp(\uui{0}(t)) \geq \frac{t}{2p} - c_0, \\
    &\implies \uui{0}(t) \geq \log\left(  \frac{t}{2p} - c_0  \right).
\end{align*}
We have the following bound on $\uui{0}(t)$,
\begin{align} \label{eq:u0-bound}
    \uui{0}(t) &\geq \log\left(  \frac{t}{2p} - c_0  \right).
\end{align}
We have established a lower bound on $\uui{0}(t)$. We will use this to lower bound the integral $\int_{0}^{t} \gamma(h) \der{h}$. We have, 
\begin{align*}
\der{\uui{0}} &= \frac{\gamma}{\cQ} \der{t} \leq \gamma \der{t}, 
\end{align*}
This implies,
\begin{align*}
 \uui{0} - \uui{0}(0) \leq \int_{0}^{t} \gamma(h) \der{h} &\implies \int_{0}^{t} \gamma(h) \der{h} \geq \uui{0}(t) - \uui{0}(0).
\end{align*}
Combining this we have the following, 
\begin{align} \label{eq:gamma-t-bound}
\int_{0}^{t} \gamma(h) \der{h} &\implies \int_{0}^{t} \gamma(h) \der{h} \geq \uui{0}(t) - \uui{0}(0) \geq \log\left(  \frac{t}{2p} - c_0  \right) - \uui{0}(0).
\end{align}
This implies that 
\begin{align}\label{eq:gamma-diverges}
    \lim_{t \to \infty} \int_{0}^{t} \gamma(h) \der{h} &\to \infty.
\end{align}
Using the bound on $\qqi{i}(t)$ derived earlier, we have,
\begin{align*}
    \qqi{i}(t) &\leq \frac{1}{1 + \frac{\delta}{p} \int_{0}^{t} \gamma(h) \der{h}} \leq \frac{p}{p + \delta \left( \log\left( t\right) - c_0  \right)}.
\end{align*}
Using the divergence property from equation~\eqref{eq:gamma-diverges}, we have,
\begin{align*}
    \lim_{t \to \infty} \qqi{i}(t) &\to 0.
\end{align*}

The only part that is left to be proved is convergence of the loss to zero. 
Note that 
\begin{align*}
    \scal{\bbeta_*}{\bbeta(t)} &= \sum_{j} \sosi{j}(t) \uui{j}(t) \leq \uui{0}(t) \sum_{j} \sosi{j}(t) \geq \frac{1}{p} \uui{0}(t).
\end{align*}
and we have already established that $\uui{0}(t) \to \infty$ as $t \to \infty$. Therefore, we have,
\begin{align*}
    \lim_{t \to \infty} \scal{\bbeta_*}{\bbeta(t)} &\to \infty.
\end{align*}
Using this in the expression of the logistic loss function in Eq.~\eqref{eq:logistic_beta}, we have,
\begin{align*}
    \lim_{t \to \infty} \bloss(\bbeta(t)) &= \lim_{t \to \infty} \log \left( 1 + \exp\left( -  \scal{\bbeta_*}{\bbeta(t)}  \right) \right) = 0.
\end{align*}    
\end{proof}

\begin{lemma}~\label{lem:log-non-maximal-rates} We will make two additional remarks regarding the theorem above which gives asymptotic bounds on the non-maximal scores and projection parameters. 
\begin{enumerate}[label=(\alph*)]
    \item The projection co-ordinates $\uui{i}(t)$ for $i \neq 0$ does not grow unbounded with time, i.e., there exists a constant $C > 0$ such that for all time $t \geq 0$, we have $\uui{i}(t) \leq C$. As $\uui{0} \to \infty$, this implies that $\VV$ is close to a rank 1 structure. 
    \item The attention scores for $i \neq 0$ satisy $\sosi{i} = O(\frac{1}{\log^2 (t)})$. 
\end{enumerate}
\end{lemma}    

\begin{proof}
From Eq.~\eqref{eq:gamma-t-bound}, we can say that $\gamma = O(1/t)$ and from the evolution of $\uui{0} - \uui{1}$, we have that, 
\begin{align*}
    d \{\uui{0} - \uui{i}\} &= \frac{\gamma}{\cQ} ( 1 - \qqi{i} ) \der{t} 
\end{align*}
As $\qqi{i} \to 0$, there exists a time $t_0$ after which we have $1 - \qqi{i} \geq 1/2$. Therefore, for $t \geq t_0$, we have,
\begin{align*}
    d \{\uui{0} - \uui{i}\} &\geq \frac{\gamma}{2p} \der{t} \\
    \uui{0}(t) - \uui{i}(t) &\geq \frac{1}{2p} \int_{t_0}^{t} \gamma(h) \der{h} + c_1 = \Omega(\log(t)),
\end{align*}

Using this in the evolution of $\qqi{i}$, we have,
\begin{align*}
    \der{\qqi{i}} &\leq - \frac{\gamma}{\cQ} \qqi{i}^2 \delta \der{t} \leq -  \frac{\gamma}{p} \qqi{i}^2 \delta \der{t}.
\end{align*}

    From the above equations~\eqref{eq:u0-bound} we have that $\uui{0}(t) \geq \log\left(  \frac{t}{2p} - c_0  \right).$ and~\eqref{eq:gamma-t-bound}, we have the following bound on $\qqi{i}(t)$,
\begin{align*}
    \der{\qqi{i}}
    &= \gamma \frac{\qqi{i}}{\cQ} \left[ - \qqi{i} \left( \uui{0} - \uui{i} \right) - \left(  1 - \qqi{i} \right)\left(  \uui{0} - \scal{\uu}{\sos} \right) \right] \der{t}. \\
    &\geq - c \, \frac{\gamma}{p} \, \qqi{i}^2 \, \log{t} \geq - c  \, \qqi{i}^2 \, \frac{\log{t}}{t}.
\end{align*}
using the asymptotic bound on $\gamma, \uui{0} - \uui{i}$. Integrating this we get that 
\begin{align*}
    \int_{1}^{\qqi{i}(t)} \frac{1}{q^2} \der{q} &\geq - c \int_{1}^{t} \frac{\log{h}}{h} \der{h}, \\
    - \left( \frac{1}{\qqi{i}(t)} - 1 \right) &\geq - c \left( \frac{(\log{t})^2}{2} - \frac{(\log{1})^2}{2} \right), \\
    \qqi{i}(t) &\leq \frac{1}{1 + c \, \frac{(\log{t})^2}{2}} = O\left( \frac{1}{\log^2(t)} \right).
\end{align*}

Now coming to the evolution of $\uui{i}$, we have,
\begin{align*}
    d{\uui{i}} &=  \gamma  \sosi{i} dt =  \gamma \frac{\qqi{i}}{\cQ} d{t} \leq \gamma \, \qqi{i} \, d{t} \leq \, c \frac{1}{t \log^2(t)} d{t}.   
\end{align*}
Integrating again, we get, 
\begin{align*}
    \uui{i}(t) - \uui{i}(t_0) &\leq c \int_{t_0}^{t} \frac{1}{h \log^2(h)} d{h}, \\
    &\leq c \left[ - \frac{1}{\log(h)} \right]_{t_0}^{t}, \\
    &\leq c \left( \frac{1}{\log(t_0)} - \frac{1}{\log(t)} \right) \leq c'.
\end{align*}
This proves the remark.
\end{proof}

\section{Linear Regression}

Here we consider the case of loss function given by the squared loss, i.e., 
\begin{align} \label{eq:sq_loss_beta}
    \bloss(\bbeta) = \| \bbeta - \bbeta_* \|^2.
\end{align}

The gradient flow dynamics for the parameters $\vv$ and $\aa$ can be written as,
\begin{align}\label{eq:gradient-flow-v-a-square}
    \mdot{\VV} &= ( \bbeta_* - \bbeta ) \sos^{\top}, \\
    \mdot{\aa} &= \left[ \diag{\sos} - \sos \sos^{\top} \right] \VV^{\top} ( \bbeta_* - \bbeta ).
\end{align}

Let $\beta_{\perp}$ be any direction be orthogonal to $\bbeta_*$. We can decompose $\vv$ as,
\begin{align*}
   \der{ \left( \beta_{\perp} \right)^{\top} \VV }  &= \beta_{\perp}^{\top} \left( \bbeta_* - \VV \sos \right) \sos^{\top} \der{t} = -  \left( \beta_{\perp} \right)^{\top} \VV \sos \sos^{\top} \der{t}. \\
\end{align*}

\Regression*
\begin{proof}
    Recall that the gradient flow dynamics of the parameters $\VV$ and $\aa$ are given by,
    \begin{align}
    \mdot{\VV} &= ( \bbeta_* - \bbeta ) \sos^{\top}, \\
    \mdot{\aa} &= \left[ \diag{\sos} - \sos \sos^{\top} \right] \VV^{\top} ( \bbeta_* - \bbeta ).
\end{align}

Let $\beta_{\perp}$ be any direction be orthogonal to $\bbeta_*$. We can decompose $\vv$ as,
\begin{align*}
   \der{ \left( \beta_{\perp} \right)^{\top} \VV }  &= \beta_{\perp}^{\top} \left( \bbeta_* - \VV \sos \right) \sos^{\top} \der{t} = -  \left( \beta_{\perp} \right)^{\top} \VV \sos \sos^{\top} \der{t}.
\end{align*}

Under this assumption, we know that the initial projection of $\VV(0)$ along $\beta_{\perp}$ is zero, i.e., $\left( \beta_{\perp} \right)^{\top} \VV(0) = 0$. Therefore, from the above equation, we have,
\begin{align*}
    \left( \beta_{\perp} \right)^{\top} \VV &= 0, \quad \forall t \geq 0.
\end{align*}
This implies that $\VV$ is rank one and $\bbeta_*$ is a left singular vector of $\VV$ and $$ \VV = \frac{\bbeta_*}{\nor{\bbeta_*}^2} \uu^{\top}. $$
Using this is the gradient flow dynamics of $\aa$, we have,
\begin{align*}
    \mdot{\aa} &= \left[ \diag{\sos} - \sos \sos^{\top} \right] \frac{\uu}{\nor{\bbeta_*}^2} \left( \bbeta_*^{\top} \bbeta_* - \bbeta_*^{\top} \bbeta \right), \\
    &= \left[ \diag{\sos} - \sos \sos^{\top} \right] \uu \left( 1 - \scal{\bbeta_*}{\bbeta} / \nor{\bbeta_*}^2 \right).     
\end{align*}
Similarly, the gradient dynamics of $\uu$ can be written as,
\begin{align*}
    \mdot{\uu} &= \nor{\bbeta_*}^2 \left( 1 - \scal{\bbeta_*}{\bbeta} / \nor{\bbeta_*}^2 \right) \sos.
\end{align*}
This finishes the proof. 
\end{proof}

\subsection{Remarks on the linear regression case}
Note that the dynamics of $\aa$ and $\uu$ are similar to the dynamics in the classification case with logistic loss. The only difference is the scaling factor $\left( 1 - \scal{\bbeta_*}{\bbeta} / \nor{\bbeta_*}^2 \right)$ which depends on the current loss. Theorem~\ref{thm:no-crossing-logitstic} implies the no-crossing property of the projections and attention scores in this case as well in the case of linear regression.

Now we will turn to the strength of the polarization, i.e., $\int \gamma(s) ds .$ Following the same approach as in Theorem~\ref{thm:sparsity-logistic}, we can establish the following result regarding the sparsity of the attention scores in the case of linear regression. From now on, assume that $\nor{\bbeta_*} = 1$ for brevity. From the Lemma~\ref{lem:convergence-general-loss}, we have that,
\begin{align*}
    \frac{ d \ell(\bbeta)}{dt} &\leq - \frac{1}{p} \nor{\nabla_{\bbeta} \ell}^2.
 \end{align*}
Using this in the context of linear regression, we get, 
\begin{align*}
    \ell(\bbeta) &= \frac{1}{2} \| \bbeta - \bbeta_* \|^2, \\
    \nabla_{\bbeta} \ell &= ( \bbeta - \bbeta_* ), 
\end{align*}
Substituting these in the convergence lemma, we have, 
\begin{align*}
     \frac{d \| \bbeta - \bbeta_* \|^2 }{dt} &\leq - \frac{1}{p} \| \bbeta - \bbeta_* \|^2. 
\end{align*}
Hence the loss converges to zero exponentially fast and note that 
$\gamma = \scal{ \bbeta_*}{\bbeta_* - \bbeta} \leq \nor{\bbeta_*} \nor{\bbeta_* - \bbeta}$. Hence that $\gamma$ also converges to zero exponentially fast. Therefore, the integral $\int_{0}^{t} \gamma(h) \der{h}$ converges to a finite value as $t \to \infty$. This implies that the attention scores do not converge to a one-hot vector in the case of linear regression.

\section{Supporting Material}

\subsection{Some properties of the value-softmax model.}

\begin{lemma}
\textbf{Invariance.} The loss function have the following invariance property:    

\begin{enumerate}[label=(\alph*)]
    \item For any $c \in \R$, we have, 
    \begin{align*}
        \loss(\VV, \aa) = \loss(\VV, \aa + c \one).
    \end{align*}
    and taking the gradient with respect to $c$ we get the following invariance in the gradients: 
    \begin{align*}
        \scal{\nabla_{\aa} \loss(\VV, \aa)}{\one} = 0.      
    \end{align*}
    \item The gradient flow dynamics has the following invariance:  
    \begin{align*}
        \sum_{i=1}^{d} \frac{d \aai{i}}{dt} = 0 \implies \sum_{i=1}^{d} \aai{i}(t) = \sum_{i=1}^{d} \aai{i}(0), \quad \forall t \geq 0.
    \end{align*}
\end{enumerate}
\end{lemma}

\begin{lemma}\label{lem:convergence-general-loss}\textbf{Convergence.}
Consider the gradient flow dynamics given by the equations~\eqref{eq:gradients-general-loss} over the loss function $\loss(\VV, \aa)$, 
\begin{align*}
    \frac{d \, \bloss(\bbeta)}{dt} &\leq - \frac{1}{d} \nor{\nabla_{\bbeta} \bloss}^2 
\end{align*} 
\end{lemma}
\begin{proof}
    With $\rrr = - \nabla_{\bbeta} \bloss $, the time derivatives of the parameters are given by, 
    \begin{align*}
        \frac{d\VV}{dt} &=  \rrr \sos^{\top},  \\
        \frac{d\sos}{dt} &=  \left( \diag{\sos} - \sos \sos^{\top} \right)^2 \VV^{\top} \rrr.
    \end{align*}
    Considering the evolution of $\bbeta = \VV \sos$, we have,
    \begin{align*}
        \frac{d \bbeta}{dt} &= \frac{d \VV \sos}{dt} = \frac{d\VV}{dt} \sos + \VV \frac{d\sos}{dt}, \\
        &= \left( \rrr \sos^{\top} \right) \sos + \VV \left( \diag{\sos} - \sos \sos^{\top} \right)^2 \VV^{\top} \rrr, \\
        &= \left[ \nor{\sos}^2 \I +  \VV \left( \diag{\sos} - \sos \sos^{\top} \right) \VV^{\top} \right] \rrr. 
    \end{align*}
    Let us denote the matrix in the brackets as $\MM(\sos, \VV)$, i.e., 
    \begin{align*}
        \MM(\sos, \VV) = \nor{\sos}^2 \I +  \VV \left( \diag{\sos} - \sos \sos^{\top} \right) \VV^{\top}.
    \end{align*}    
    Using this notation, we can write the evolution of $\bbeta$ as follows:
    \begin{align*}
        \frac{d \bbeta}{dt} &=  \MM(\sos, \VV) \rrr = - \MM(\sos, \VV) \nabla_{\bbeta} \bloss. 
    \end{align*}
    Considering the time derivative of the loss function $\bloss(\bbeta)$, we have,
    \begin{align*}
        \frac{d \bloss(\bbeta)}{dt} &= \scal{\nabla_{\bbeta} \bloss}{\frac{d \bbeta}{dt}} = - \scal{\nabla_{\bbeta} \bloss \, \, }{\, \, \MM(\sos, \VV) \, \nabla_{\bbeta} \bloss}. 
    \end{align*}
    Now we will show that the matrix $\MM(\sos, \VV)$ is positive definite. First, 
    \begin{align*}
        \nor{\sos}^2  \geq \sum_{i=1}^{d} \sosi{i}^2 \geq \frac{1}{d} \left( \sum_{i=1}^{d} \sosi{i} \right)^2 = \frac{1}{d}.
    \end{align*}
    The next term is positive semi-definite since for any vector $\uu \in \R^{d}$, we have,
    \begin{align*}
        \uu^{\top} \VV \left( \diag{\sos} - \sos \sos^{\top} \right) \VV^{\top} \uu &= \left( \VV^{\top} \uu \right)^{\top} \left( \diag{\sos} - \sos \sos^{\top} \right) \left( \VV^{\top} \uu \right), \\
        &= \mu^{\top} \left( \diag{\sos} - \sos \sos^{\top} \right) \mu, \quad \text{ where } \mu = \VV^{\top} \uu, \\
        &= \sum_{i=1}^{d} \sosi{i} \mu_i^2 - \left( \sum_{i=1}^{d} \sosi{i} \mu_i \right)^2 \geq 0,
    \end{align*}
    where the last inequality follows from the Cauchy-Schwarz inequality. Hence we have,
    \begin{align*}
        \uu^{\top} \MM(\sos, \VV) \uu \geq \frac{1}{d} \nor{\uu}^2, \quad \forall \, \uu \in \R^{d}.
    \end{align*}
    This shows that the matrix $\MM(\sos, \VV)$ is positive definite. Using the positive definiteness of $\MM(\sos, \VV)$, we have,
    \begin{align*}
        \frac{d \bloss(\bbeta)}{dt} &\leq - \frac{1}{d} \nor{\nabla_{\bbeta} \bloss}^2 \leq 0.
    \end{align*}
    This shows that the loss function $\bloss(\bbeta)$ is non-increasing along the gradient flow dynamics of $\VV$ and $\sos$.
\end{proof}

\begin{property} \label{prop:jacobians}
For the value-softmax model, i.e., $\bbeta = \VV \sf(\aa)$, this property summarizes the jacobians of the different components.
    \begin{itemize}
        \item The jacobian of the softmax function $\sf: \R^{d} \to \Delta^{d}$ :
        \begin{align}\label{eq:softmax_jacobian}
            \jacb{\aa}{\sf(\aa)} = \diag{\sf(\aa)} - \sf(\aa) \sf(\aa)^{\top}. 
        \end{align}
        With $\sos = \sf(\aa)$, the jacobian can be written as
        \begin{align}\label{eq:s_jacobian}
            \jacb{\aa}{\sos} = \diag{\sos} - \sos \sos^{\top}. 
        \end{align}
        \item The jacobian with respect to the $\bbeta$ with respect to $\VV$ and $\aa$ are given by
        \begin{align}
            \jacb{\VV}{\bbeta} &= \sos^{\top} \otimes \I_{d}, \\
            \jacb{\sos}{\bbeta} &= \VV, \\
            \jacb{\aa}{\bbeta} &= \VV \jacb{\aa}{\sos} = \VV \left( \diag{\sos} - \sos \sos^{\top} \right).
        \end{align}
        It is needed to be clarified that the jacobian $\jacb{\VV}{\bbeta}$ is a matrix of size $d \times d^2$ and it is constructed in the following way:
        \begin{align*}
            \jacb{\VV}{\bbeta} = \begin{bmatrix}
                \jacb{\vvc{1}}{\bbeta} & \jacb{\vvc{2}}{\bbeta} & \ldots & \jacb{\vvc{d}}{\bbeta}
            \end{bmatrix} = \sos^{\top} \otimes \I_{d}.
        \end{align*}
        where each row is $ \jacb{\mathrm{vec}{(\VV)}}{\bbeta_i} $ . 
    \end{itemize}
\end{property}

\subsection{Normalization} 
Here, we will consider the case of general normalization functions and establish that the no-crossing property holds for the logistic loss as well as the regression loss function . In particular, 
consider a general monotonically increasing differentiable function $f$, and define the normalization $\sf_{f}: \R^{p} \to \R^{p}$ for any $a$ in $\R^{p}$ as
\begin{align*}
  \sf_{f}(a)_i = \frac{f(a_i)}{\sum\limits_{j=1}^{p} f(a_j)}.
\end{align*}
For the logistic loss function, the gradient flow dynamics is given by
\begin{align} \label{eq:general-normalization-logistic} 
  \der{\uu} &=  \gamma(\bbeta) \, \|\bbeta_*\|^2   \sos_{f} \, \der{t}, \\
  \der{\aa} &=  \gamma(\bbeta) \, \frac{f'(\aa)}{\sum\limits_{j=1}^{p} f(\aa_j)} \odot ( \uu - \scal{\uu}{\sf_f} \one ) \, \der{t}.
\end{align}

\begin{lemma}
    Consider the gradient flow dynamics given by the equation Eq.~\eqref{eq:general-normalization-logistic} over the logistic loss function $\bloss(\bbeta)$. For the initialization given by the Assumption~\ref{ass:init}, the following no-crossing property holds for all time $t \geq 0$,
    \begin{align*}
        \uui{0}(t) > \uui{1}(t) > \uui{2}(t) \ldots > \uui{p-1}(t), \\
        \aai{0} (t) \geq  \aai{1} (t) \geq \aai{2} (t) \ldots \geq \aai{p-1} (t).
     \end{align*}    
\end{lemma}
\begin{proof}
    Let $\bF = \sum_{j} f(\aa_j)$ and let $$ G(s) = \int_{0}^{s} \frac{1}{f'(v)} dv. $$ Using this to we rewrite the co-ordinate wise dynamics as 
    \begin{align*}
        d\uui{i} &= \frac{\gamma}{\bF} f(\aai{i}) dt, \\ 
        d\aai{i} &= \frac{\gamma}{\bF} f'(\aai{i}) \left( \uui{i} - \scal{\uu}{\sos_f} \right) dt, \\
        \frac{d \aai{i}}{f'(\aai{i})} &= \frac{\gamma}{\bF} \left( \uui{i} - \scal{\uu}{\sf_f} \right) dt. \\
        d{ G(a_i) } &= \frac{\gamma}{\bF} \left( \uui{i} - \scal{\uu}{\sf_f} \right) dt, \\
    \end{align*}
For any $i,j$, taking the differences, 
    \begin{align*}
        d \left( G(\aai{i}) - G(\aai{j}) \right) &= \frac{\gamma}{\bF} \left( \uui{i} - \uui{j} \right) dt. \\
        d \left( \uui{i} - \uui{j} \right) &= \frac{\gamma}{\bF} \left( f(\aai{i}) - f(\aai{j}) \right) dt.
    \end{align*}

    We can take a potential function $\Phi_{ij} = \left( G(\aai{i}) - G(\aai{j}) \right)\left( \uui{i} - \uui{j} \right)$, using this we have, 
    \begin{align*}
        d \Phi_{ij} &= \left( \uui{i} - \uui{j} \right) d \left( G(\aai{i}) - G(\aai{j}) \right) + \left( G(\aai{i}) - G(\aai{j}) \right) d \left( \uui{i} - \uui{j} \right), \\
        &= \frac{\gamma}{\bF} \left[ \left( \uui{i} - \uui{j} \right)^2 + \left( G(\aai{i}) - G(\aai{j}) \right) \left( f(\aai{i}) - f(\aai{j}) \right) \right] dt. 
    \end{align*}
    Note that $f$ increases monotonically, hence $f'(a) > 0$ everywhere, which implies $1/f'(a) > 0$ everywhere and G also increases monotonically. We assume that it does not hit $0$ in finite time, so the term $\left( G(\aai{i}) - G(\aai{j}) \right) \left( f(\aai{i}) - f(\aai{j}) \right)$ is always non-negative. This implies that $d \Phi_{ij} \geq 0$, ensuring that the potential function $\Phi_{ij}$ does not decrease over time. Consequently, the no-crossing property holds for all time $t \geq 0$.
\end{proof}




\section{Additional Experimental Details}
\label{app:experimental-details}

\subsection{Induction heads acting as attention sinks}

\paragraph{Task format.} The structure of each object in the dataset is illustrated in Figure \ref{fig:induction_task_format}. Each object consists of a beginning-of-sequence token, $k$ randomly sampled tokens from the vocabulary $t_1,\dots,t_k$, a token $t_i'$ uniquely corresponding to a randomly chosen $t_i$, followed by $t_{i+1}$. We experimented both with $t_i' = t_i$ and $t_i'$ being a different symbol, observing that the latter typically leads to less attention sinks forming.

The loss is calculated as $L_\mathrm{total} = \lambda \cdot (1/k) L_\mathrm{random} + L_\mathrm{induction}$. We tested $\lambda = 0$ and $\lambda = 0.2$. In the former case, the model is only penalized on incorrect predictions of the last token and not incentivized to keep a constant uniform prediction on random tokens. This leads to all heads duplicating each other and no almost no attention sinks for any attention type.

\begin{figure}[h]
  \centering
  \includegraphics[width=0.5\textwidth]{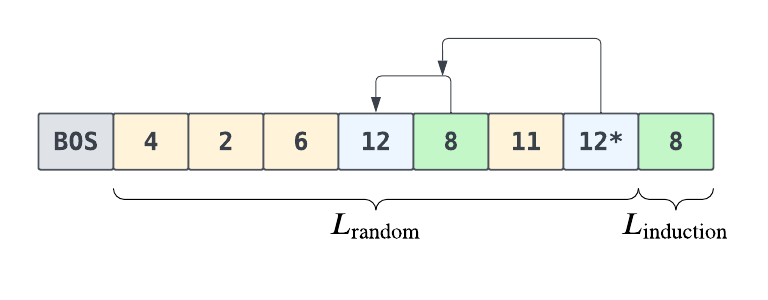}
  \caption{Task format in the induction heads experiment. The arrows illustrate the induction head mechanism.}
  \label{fig:induction_task_format}
\end{figure}

\begin{table}[h]
  \centering
  \renewcommand{\arraystretch}{1.15}
  \setlength{\tabcolsep}{10pt}
  \begin{tabular}{l c l c}
    \hline
    \textbf{Hyperparameter} & \textbf{Value} & \textbf{Hyperparameter} & \textbf{Value} \\
    \hline
    Layers           & 2   & Epochs     & 10 \\
    Hidden dim       & 128 & Train size         & 100{,}000 \\
    Feedforward dim  & 512 & Eval size          & 1{,}000 \\
    Vocab size       & $51^*$  & Sequence length & 13 \\
    Causal           & \texttt{true} & Batch size     & 256 \\
    Use LayerNorm    & \texttt{true} & Use MLP        & \texttt{true} \\
    \hline
    Heads            & $\bf \{1,2,4,8\}$ & $t_i' = t_i$ & $\{\texttt{true},\textbf{\texttt{false}}\}$ \\
    Learning rate    & $\bf \{10^{-2},10^{-3},10^{-4}\}$ & $\lambda$ & $\{0.0,\bf{0.2}\}$ \\
    Attention type   & all in Table \ref{table:attention-types} & Positional type  & \{\texttt{concat}, \textbf{\texttt{additive}}, \texttt{no}\} \\
    Seed             & $\bf \{0,1,2,3,4\}$ & & \\
    \hline
  \end{tabular}
  \vspace{3pt}
  \caption{Configuration and hyperparameters for the induction experiment. The bold values indicate the setup shown in Figures \ref{fig1} and \ref{fig:main-exp-induction}. Results for other setups are additionally presented in Appendix \ref{app:results-induction} Vocabulary size is increased by 50 if $t_i' \ne t_i$.}
  \label{tab:induction-config-sweep}
\end{table}

\paragraph{Hyperparameters.}

In this experiment, the model is a standard 2-layer $H$-head causal Transformer with pre-normalization, similar to GPT-2 \cite{radford2019language}. The full hyperparameter and configuration details are available in Table \ref{tab:induction-config-sweep}.

\paragraph{Metrics.}

We discard all runs that did not achieve induction accuracy above $95\%$. A sink score for each head is measured as an average proportion of attention weight assigned to the BOS token on the positions where no informed prediction is possible. We consider a head \emph{a sink head} if it's sink score is higher than 0.9.
The computation can be expressed via the following pseudocode:
{\small \begin{verbatim}
    A = attention_matrix  # [layer, head, sample, query_seq_len, key_seq_len]

    # index 0 is removed because it's BOS. 
    # indices -1 and -2 are removed because prediction is not random.
    bos_score[L, H] = A[L, H, :, 1 : -2, 0]
    total_weight[L, H] = A[L, H, :, 1 : -2, :].sum(axis=-1)

    # clipping affects attention types without positivity. doesn't change is_sink.
    sink_score[L, H] = (bos_score[L, H] / total_weight[L, H]).mean().clip(0, 1) 
    is_sink[L, H] = (sink_score[L, H] > 0.9)
\end{verbatim}}

We only plot the results from the second layer, as all heads in the first layer typically attend to the previous token.

\subsection{Difference in attention sparsity between pretrained softmax and sigmoid LLMs}

We use the weights and the codebase\footnote{https://github.com/apple/ml-sigmoid-attention} for the pretrained 7B softmax and sigmoid models released by \citet{ramapuramtheory}. We run both models on 1000 samples from a subset\footnote{https://huggingface.co/datasets/NeelNanda/pile-10k} of the Pile dataset \citep{gao2020pile} with the maximum sequence length of 256. We measure the sparsity score of each head as the average proportion of total attention weight assigned to the max-logit token. In pseudocode,
{\small \begin{verbatim}
    A = attention_matrix  # [layer, head, sample, query_seq_len, key_seq_len]
    max_attn_weight = A.max(axis=-1)
    total_attn_weight = A.sum(axis=-1)
    sparsity_score[L, H] = (max_attn_weight / total_attn_weight).mean()
\end{verbatim}}

In Figure \ref{fig:main-llm-pretrained}, we plot the distribution of sparsity scores for all heads in softmax and sigmoid models.

\subsection{Emerging imbalance in token influence on classifier's predictions}

\paragraph{Task format.} The task is sequence classification into $K$ labels. The vocabulary is split into $K$ disjoint groups $V^{(1)},\dots,V^{(K)}$. Each object with a label $k$ is a random sequence of tokens from the $k$-th group: $x_i = v^{(k_i)}_{j_{i1}},\dots,v^{(k_i)}_{j_{in}}, y_i=k$. \texttt{CLS} token is appended to the input. Knowledge of any of the input tokens is sufficient to achieve perfect in-distribution accuracy. The task is illustrated in Figure \ref{fig:cls_task_format} (left).

\begin{figure}[h]
  \centering
  \includegraphics[width=1.0\textwidth]{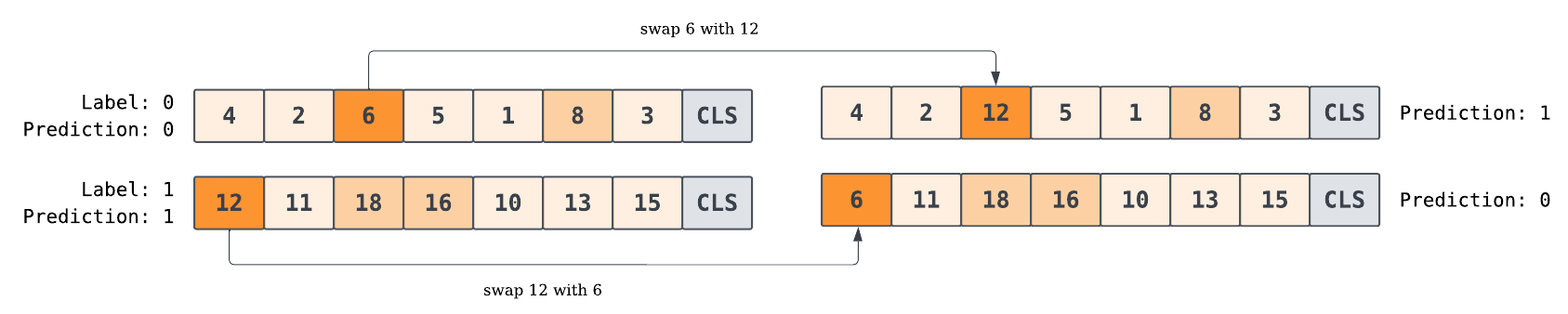}
  \caption{Task format in the classification experiment and the token swap procedure. Cell brightness illustrates the attention weight assigned to each token.}
  \label{fig:cls_task_format}
\end{figure}

\begin{table}[h]
  \centering
  \renewcommand{\arraystretch}{1.15}
  \setlength{\tabcolsep}{10pt}
  \begin{tabular}{l c l c}
    \hline
    \textbf{Hyperparameter} & \textbf{Value} & \textbf{Hyperparameter} & \textbf{Value} \\
    \hline
    Layers           & $\{{\bf1},2\}$ & Epochs         & 10 \\
    Hidden dim       & 128          & Train size     & 10{,}000 \\
    Feedforward dim  & 512          & Eval size      & 1{,}000 \\
    Vocab size       & $\{\bf{41},81,161\}$ & Sequence length & 11 \\
    Labels & 4     & Causal        & \texttt{false} \\
    Batch size       & 256          & Positional type & \texttt{no} \\
    \hline
    Heads            & $\{{\bf 1},2,4\}$ & Use MLP        & $\{\textbf{\texttt{true}},\texttt{false}\}$ \\
    Use LayerNorm    & $\{\textbf{\texttt{true}},\texttt{false}\}$ & Learning rate & $\bf\{10^{-2},10^{-3},10^{-4}\}$ \\
    Attention type   & all in Table \ref{table:attention-types} & Seed & $\bf\{0,1,2,3,4\}$ \\
    \hline
  \end{tabular}
  \vspace{3pt}
  \caption{Configuration and hyperparameters for the classification experiment. The bold values indicate parameters used for the setup in Figures \ref{fig:main-exp-classification} and \ref{fig:exp-classifcation-attentions}. Results for other setups are additionally provided in Appendix \ref{app:results-cls}}
  \label{tab:classification-config-sweep}
\end{table}

\paragraph{Hyperparameters.}

We use the same architecture as in the induction head experiment. However, we do not use causal masking or positional encodings as the input is effectively a set. We experiment with different numbers of layers and heads, as well as active and inactive MLPs and LayerNorm layers. Full details are provided in Table \ref{tab:classification-config-sweep}.

\paragraph{Metrics.}

We calculate the out-of-distribution flip rate, presented in Figure \ref{fig:main-exp-classification}, as follows. We repeatedly pick two random sequences with different labels from the evaluation set and run a forward pass on them, extracting the attention weights with the \texttt{CLS} token as a query. Then we swap the max-logit tokens for the head $(L, H)$ and check if the model changed the prediction. In pseudocode,
{\small \begin{verbatim}
    X1, X2, Y1, Y2 = pick_samples_with_different_labels()
    
    A1 = attention_matrix(X1)  # [layer, head, query_seq_len, key_seq_len]
    A2 = attention_matrix(X2)

    # obtaining the max-logit tokens excluding CLS itself
    ranking_1[L, H] = A1[L, H,-1,:-1].argsort(reverse=True)
    ranking_2[L, H] = A2[L, H,-1,:-1].argsort(reverse=True)

    # max-logit token from label 2 + others from label 1
    adversarial_1[L, H] = [ranking_2[L, H][0]] + ranking_1[L, H][1:]

    # max-logit token from label 1 + others from label 2
    adversarial_2[L, H] = [ranking_1[L, H][0]] + ranking_2[L, H][1:]

    prediction_adv_1[L, H] = model(adversarial_1[L, H])
    prediction_adv_2[L, H] = model(adversarial_2[L, H])

    # flipped if swapping one token changes the prediction
    flipped_1[L, H] = (prediction_adv_1[L, H] == Y2)
    flipped_2[L, H] = (prediction_adv_2[L, H] == Y1)
    flip_rate[L, H] = (flip_1 + flip_2) / 2
\end{verbatim}}

Figure \ref{fig:cls_task_format} illustrates this procedure. We report \emph{the maximum flip rate} over all heads and layers.

\section{Additional Experimental Results}

\subsection{Value-Softmax Model}

\begin{figure}[h]
  \centering
  \includegraphics[width=1.0\textwidth]{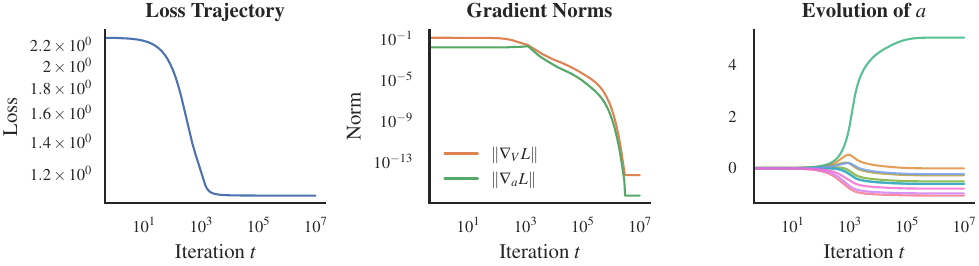}
  \caption{Training dynamics of Value-Softmax model with KL-divergence as loss function.}
  \label{fig:kl}
\end{figure}

\begin{figure}[h]
  \centering
  \includegraphics[width=1.0\textwidth]{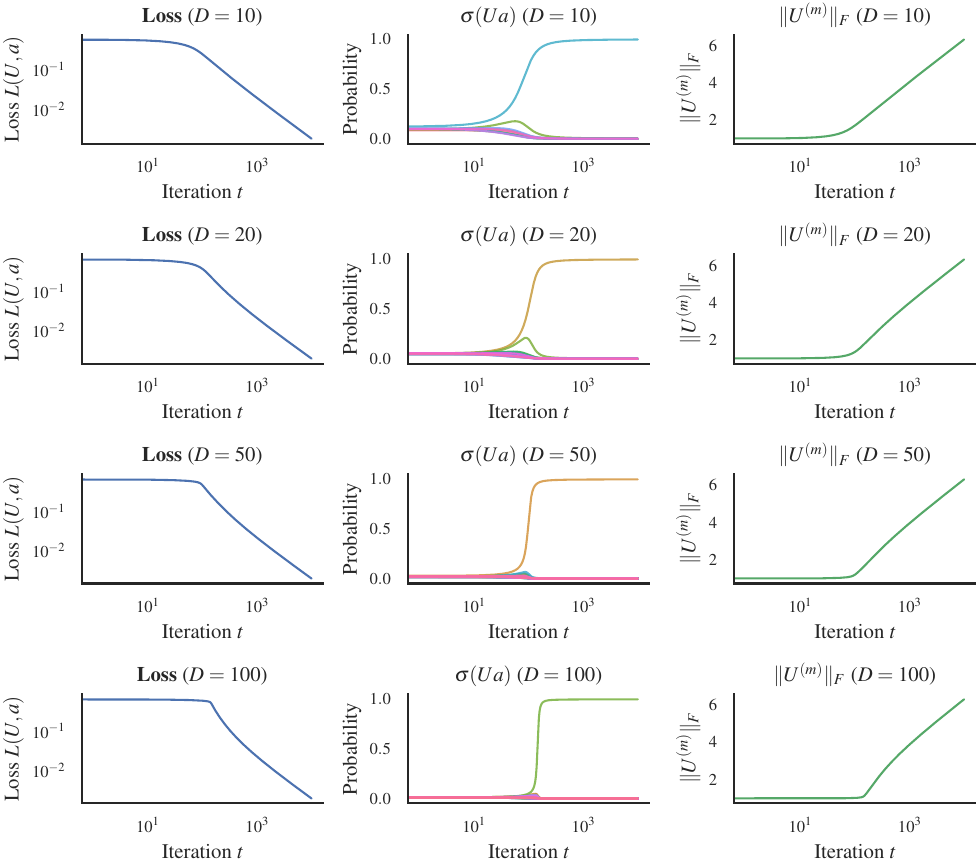}
  \caption{Training dynamics of tied Value-Softmax model ($U \sigma(Ua)$) with logistic loss for different dimensionality $D$. We denote the column of $U$ corresponding to the largest element of $a$ as $U^{(m)}$. In all cases, a sparse solution is found, and the norm of $U^{(m)}$ diverges.} 
  \label{fig:massive_activations}
\end{figure}

\begin{figure}[h]
  \centering
  \includegraphics[width=1.0\textwidth]{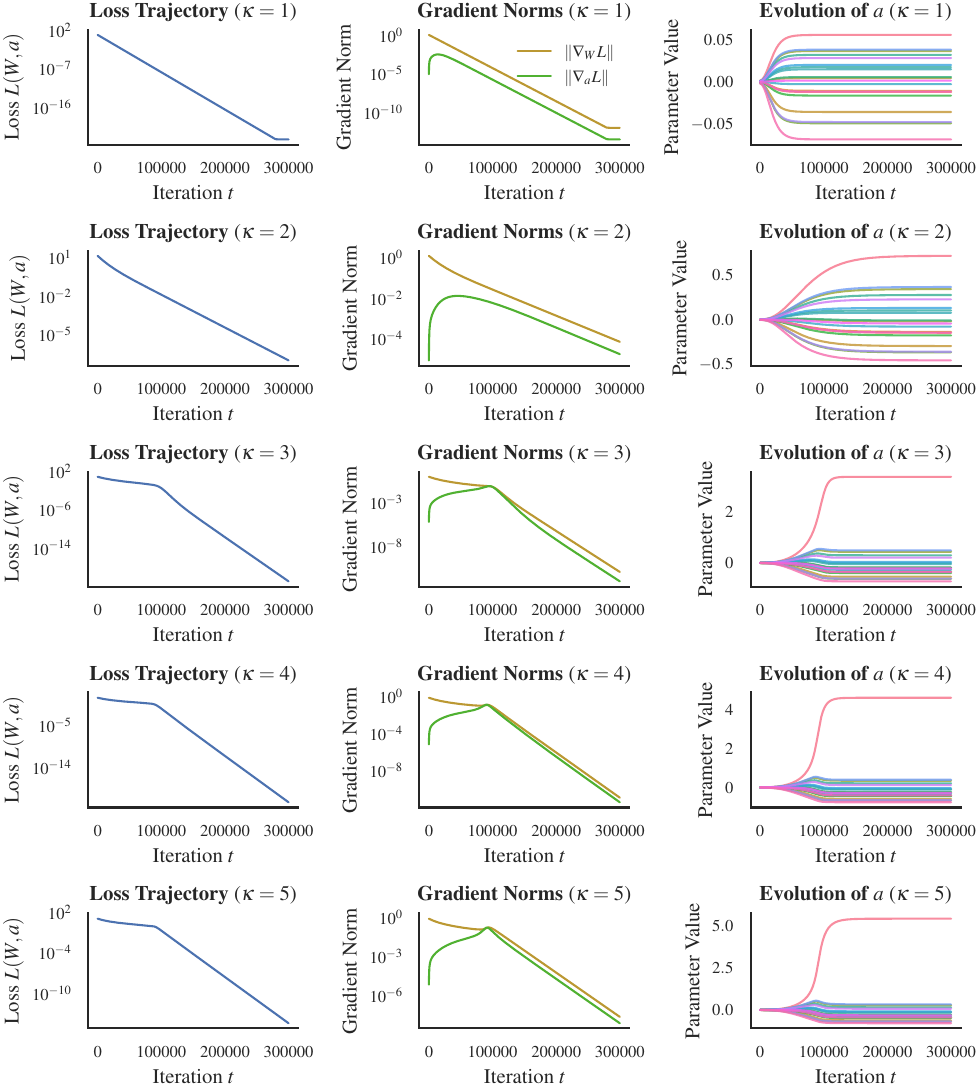}
  \caption{Training dynamics in regression problems with condition number $\kappa$ from 1 to 5. Higher condition number leads to slower optimization and stronger sparsity.}
  \label{fig:condition_numbers}
\end{figure}

\begin{figure}[t]
  \centering
  \begin{subfigure}[t]{0.46\textwidth}
    \centering
    \includegraphics[width=\textwidth]{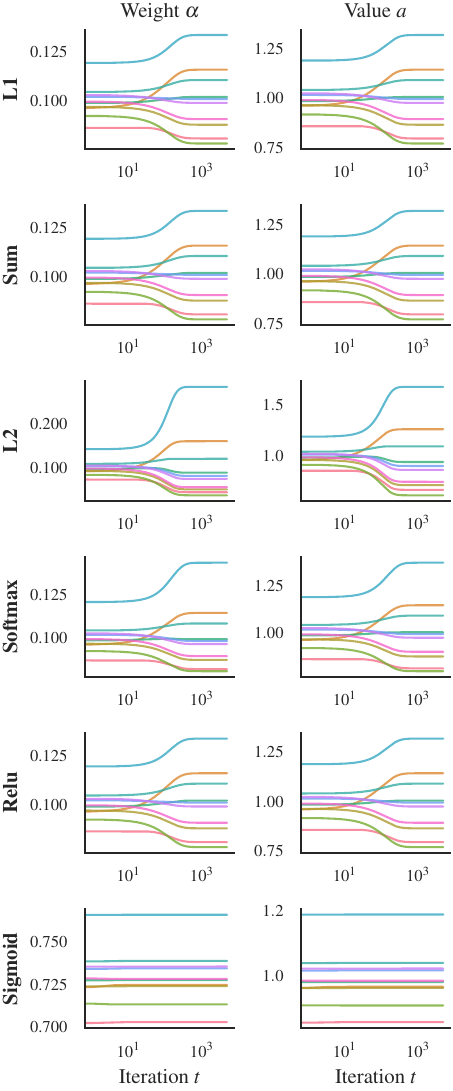}
    \caption{Square loss.}
  \end{subfigure}
  \hfill
  \begin{subfigure}[t]{0.46\textwidth}
    \centering
    \includegraphics[width=\textwidth]{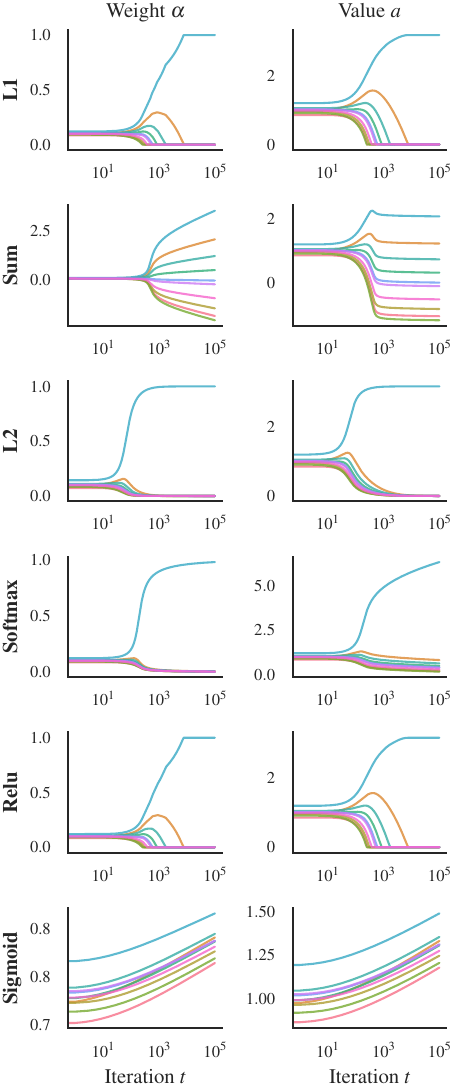}
    \caption{Logistic loss.}
  \end{subfigure}

  \caption{Training dynamics of value-softmax model for square and logistic losses and various activation functions. Note that square loss does not give rise to sparse solutions. Additionally, sum and sigmoid do not converge to sparse solutions for logistic loss as well.}
  \label{fig:toy-many-plots}
\end{figure}

\FloatBarrier
\subsection{Induction heads acting as attention sinks}
\label{app:results-induction}


\begin{figure*}[h]
  \centering

  \begin{subfigure}[t]{0.46\textwidth}
    \centering
    \includegraphics[width=\textwidth]{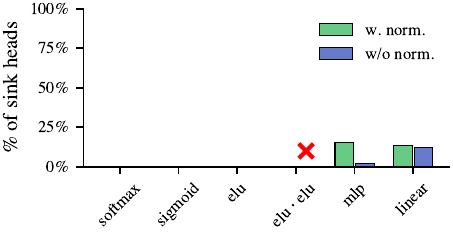}
    \caption{additive, $t_i = t_i'$}
  \end{subfigure}\hfill
  \begin{subfigure}[t]{0.46\textwidth}
    \centering
    \includegraphics[width=\textwidth]{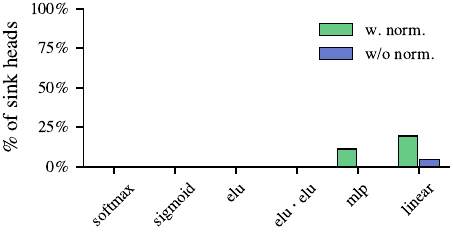}
    \caption{additive, $t_i \ne t_i'$}
  \end{subfigure}

  \vspace{2mm}

  \begin{subfigure}[t]{0.46\textwidth}
    \centering
    \includegraphics[width=\textwidth]{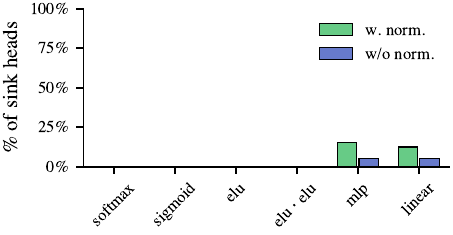}
    \caption{concatenated, $t_i = t_i'$}
  \end{subfigure}\hfill
  \begin{subfigure}[t]{0.46\textwidth}
    \centering
    \includegraphics[width=\textwidth]{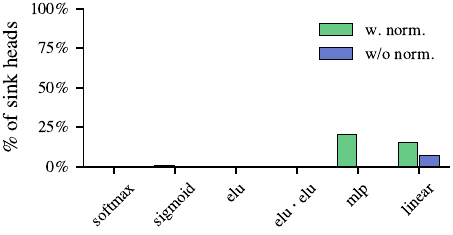}
    \caption{concatenated, $t_i \ne t_i'$}
  \end{subfigure}

  \vspace{2mm}

  \begin{subfigure}[t]{0.46\textwidth}
    \centering
    \includegraphics[width=\textwidth]{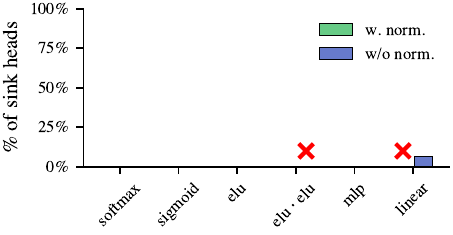}
    \caption{no, $t_i = t_i'$}
  \end{subfigure}\hfill
  \begin{subfigure}[t]{0.46\textwidth}
    \centering
    \includegraphics[width=\textwidth]{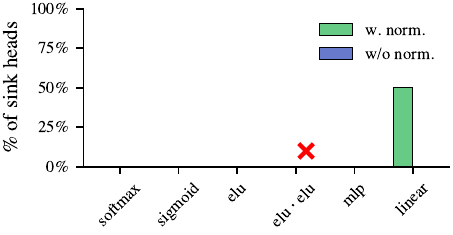}
    \caption{no, $t_i \ne t_i'$}
  \end{subfigure}

  \caption{Sink proportions in induction-trained Transformer with $\lambda = 0$. Each row represents a type of positional encodings. Setups that result in training failures are marked by red crosses. We observe that under this setup, most heads duplicate each other and do not form sinks.}
  \label{fig:sink-proportions-aux0}
\end{figure*}

\begin{figure*}[h]
  \centering

  \begin{subfigure}[t]{0.46\textwidth}
    \centering
    \includegraphics[width=\textwidth]{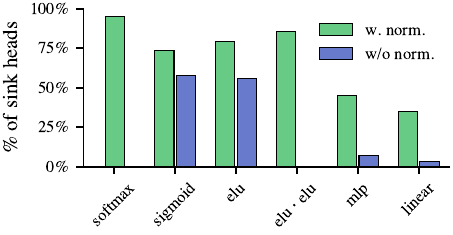}
    \caption{additive, $t_i = t_i'$}
  \end{subfigure}\hfill
  \begin{subfigure}[t]{0.46\textwidth}
    \centering
    \includegraphics[width=\textwidth]{figs/sink_proportions_aux0.2_shift_add.pdf}
    \caption{additive, $t_i \ne t_i'$}
  \end{subfigure}

  \vspace{2mm}

  \begin{subfigure}[t]{0.46\textwidth}
    \centering
    \includegraphics[width=\textwidth]{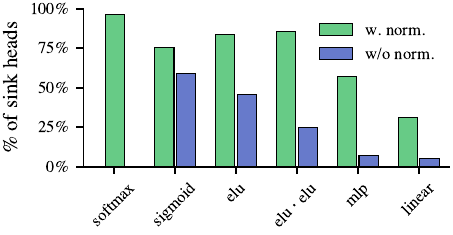}
    \caption{concatenated, $t_i = t_i'$}
  \end{subfigure}\hfill
  \begin{subfigure}[t]{0.46\textwidth}
    \centering
    \includegraphics[width=\textwidth]{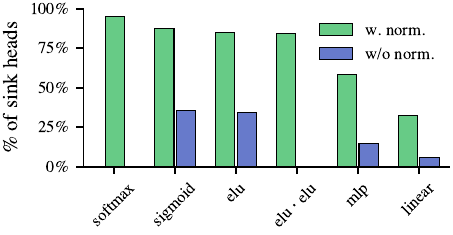}
    \caption{concatenated, $t_i \ne t_i'$}
  \end{subfigure}

  \vspace{2mm}

  \begin{subfigure}[t]{0.46\textwidth}
    \centering
    \includegraphics[width=\textwidth]{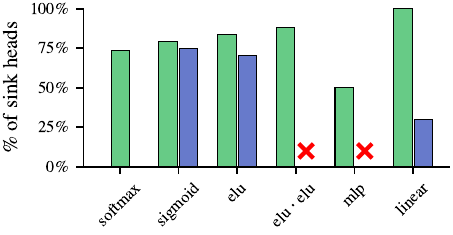}
    \caption{no, $t_i = t_i'$}
  \end{subfigure}\hfill
  \begin{subfigure}[t]{0.46\textwidth}
    \centering
    \includegraphics[width=\textwidth]{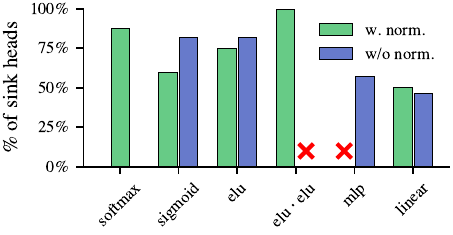}
    \caption{no, $t_i \ne t_i'$}
  \end{subfigure}

  \caption{Sink proportions in induction-trained Transformer with $\lambda = 0.2$. Each row represents a type of positional encodings. Setups that result in training failures are marked by red crosses. Interestingly, we observe that setups with no explicit positional encodings lead to an increased formation of attention sinks. Note that in this case the model has access to positional information through causal masking, so the task is still solvable.}
  \label{fig:sink-proportions-aux02}
\end{figure*}

\begin{figure*}[h]
  \centering

  \begin{subfigure}[t]{0.46\textwidth}
    \centering
    \includegraphics[width=\textwidth]{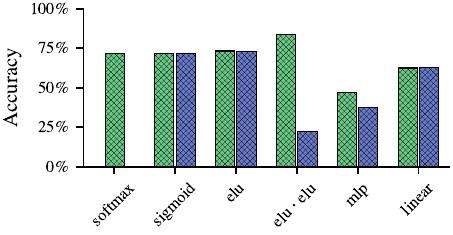}
    \caption{additive, $t_i = t_i'$}
  \end{subfigure}\hfill
  \begin{subfigure}[t]{0.46\textwidth}
    \centering
    \includegraphics[width=\textwidth]{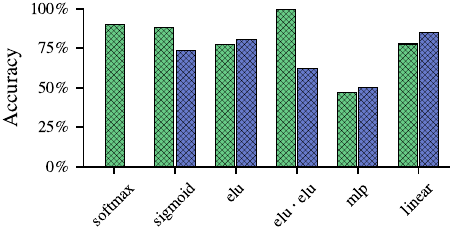}
    \caption{additive, $t_i \ne t_i'$}
  \end{subfigure}

  \vspace{2mm}

  \begin{subfigure}[t]{0.46\textwidth}
    \centering
    \includegraphics[width=\textwidth]{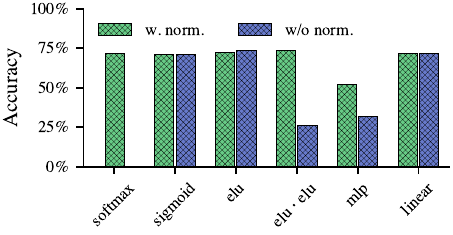}
    \caption{concatenated, $t_i = t_i'$}
  \end{subfigure}\hfill
  \begin{subfigure}[t]{0.46\textwidth}
    \centering
    \includegraphics[width=\textwidth]{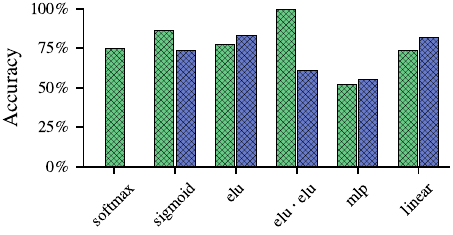}
    \caption{concatenated, $t_i \ne t_i'$}
  \end{subfigure}

  \vspace{2mm}

  \begin{subfigure}[t]{0.46\textwidth}
    \centering
    \includegraphics[width=\textwidth]{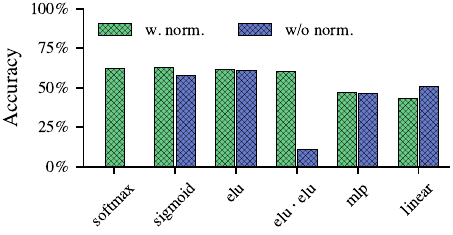}
    \caption{no, $t_i = t_i'$}
  \end{subfigure}\hfill
  \begin{subfigure}[t]{0.46\textwidth}
    \centering
    \includegraphics[width=\textwidth]{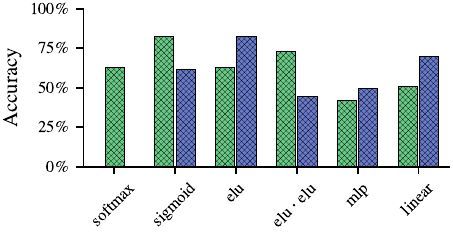}
    \caption{no, $t_i \ne t_i'$}
  \end{subfigure}

  \caption{Average induction accuracy for models trained with $\lambda = 0$. Each row represents a type of positional encodings.}
  \label{fig:accuracy-aux0}
\end{figure*}

\begin{figure*}[h]
  \centering

  \begin{subfigure}[t]{0.46\textwidth}
    \centering
    \includegraphics[width=\textwidth]{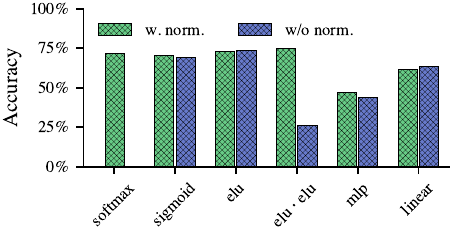}
    \caption{additive, $t_i = t_i'$}
  \end{subfigure}\hfill
  \begin{subfigure}[t]{0.46\textwidth}
    \centering
    \includegraphics[width=\textwidth]{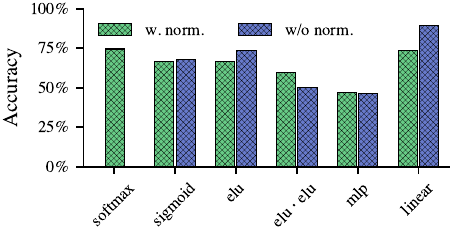}
    \caption{additive, $t_i \ne t_i'$}
  \end{subfigure}

  \vspace{2mm}

  \begin{subfigure}[t]{0.46\textwidth}
    \centering
    \includegraphics[width=\textwidth]{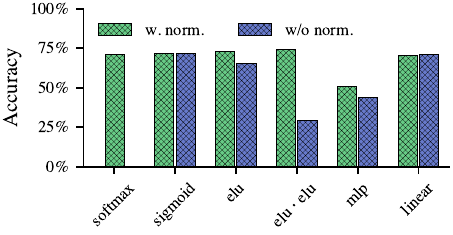}
    \caption{concatenated, $t_i = t_i'$}
  \end{subfigure}\hfill
  \begin{subfigure}[t]{0.46\textwidth}
    \centering
    \includegraphics[width=\textwidth]{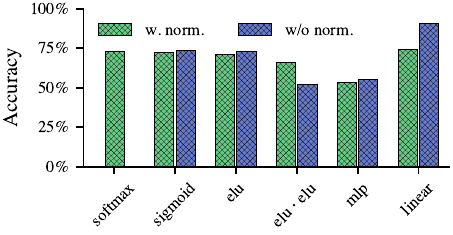}
    \caption{concatenated, $t_i \ne t_i'$}
  \end{subfigure}

  \vspace{2mm}

  \begin{subfigure}[t]{0.46\textwidth}
    \centering
    \includegraphics[width=\textwidth]{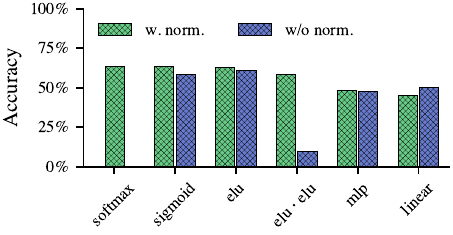}
    \caption{no, $t_i = t_i'$}
  \end{subfigure}\hfill
  \begin{subfigure}[t]{0.46\textwidth}
    \centering
    \includegraphics[width=\textwidth]{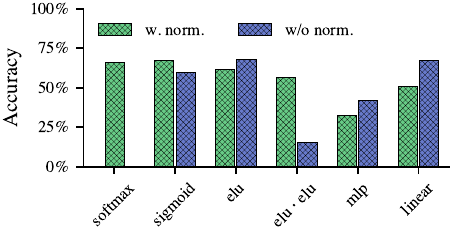}
    \caption{no, $t_i \ne t_i'$}
  \end{subfigure}

  \caption{Average induction accuracy for models trained with $\lambda = 0.2$. Each row represents a type of positional encodings.}
  \label{fig:accuracy-aux02}
\end{figure*}

\FloatBarrier
\subsection{Emerging imbalance in token influence on classifier's predictions}
\label{app:results-cls}


\begin{figure*}[h]
  \centering

  \begin{subfigure}[t]{0.46\textwidth}
    \centering
    \includegraphics[width=\textwidth]{figs/fliprate_vocab41_mlp_ln.pdf}
    \caption{MLP, LN}
  \end{subfigure}\hfill
  \begin{subfigure}[t]{0.46\textwidth}
    \centering
    \includegraphics[width=\textwidth]{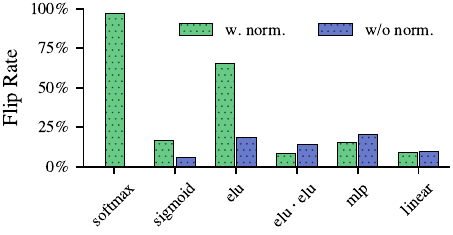}
    \caption{MLP, no LN}
  \end{subfigure}

  \vspace{2mm}

  \begin{subfigure}[t]{0.46\textwidth}
    \centering
    \includegraphics[width=\textwidth]{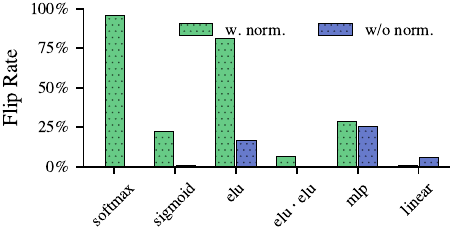}
    \caption{no MLP, LN}
  \end{subfigure}\hfill
  \begin{subfigure}[t]{0.46\textwidth}
    \centering
    \includegraphics[width=\textwidth]{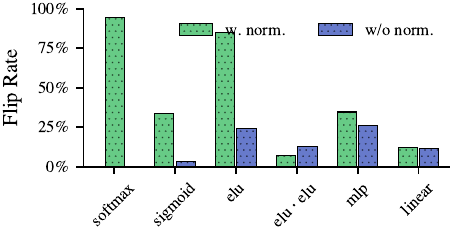}
    \caption{no MLP, no LN}
  \end{subfigure}

  \caption{Flip-rate plots for vocabulary size $|\mathcal{V}|=41$. Rows: MLP vs. no-MLP. Columns: LayerNorm vs. no-LayerNorm.}
  \label{fig:fliprate-vocab41}
\end{figure*}

\begin{figure*}[h]
  \centering

  \begin{subfigure}[t]{0.46\textwidth}
    \centering
    \includegraphics[width=\textwidth]{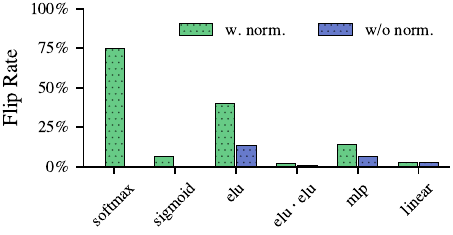}
    \caption{MLP, LN}
  \end{subfigure}\hfill
  \begin{subfigure}[t]{0.46\textwidth}
    \centering
    \includegraphics[width=\textwidth]{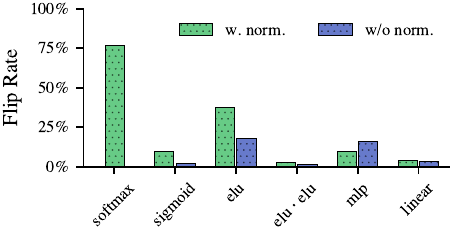}
    \caption{MLP, no LN}
  \end{subfigure}

  \vspace{2mm}

  \begin{subfigure}[t]{0.46\textwidth}
    \centering
    \includegraphics[width=\textwidth]{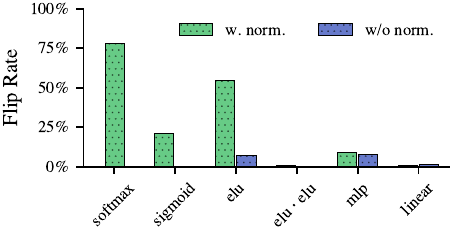}
    \caption{no MLP, LN}
  \end{subfigure}\hfill
  \begin{subfigure}[t]{0.46\textwidth}
    \centering
    \includegraphics[width=\textwidth]{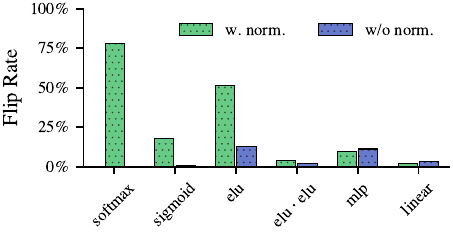}
    \caption{no MLP, no LN}
  \end{subfigure}

  \caption{Flip-rate plots for vocabulary size $|\mathcal{V}|=81$. Rows: MLP vs. no-MLP. Columns: LayerNorm vs. no-LayerNorm.}
  \label{fig:fliprate-vocab81}
\end{figure*}

\begin{figure*}[h]
  \centering

  \begin{subfigure}[t]{0.46\textwidth}
    \centering
    \includegraphics[width=\textwidth]{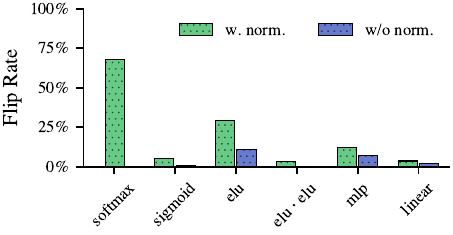}
    \caption{MLP, LN}
  \end{subfigure}\hfill
  \begin{subfigure}[t]{0.46\textwidth}
    \centering
    \includegraphics[width=\textwidth]{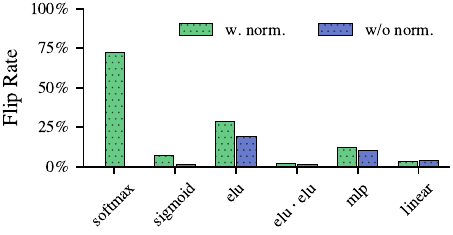}
    \caption{MLP, no LN}
  \end{subfigure}

  \vspace{2mm}

  \begin{subfigure}[t]{0.46\textwidth}
    \centering
    \includegraphics[width=\textwidth]{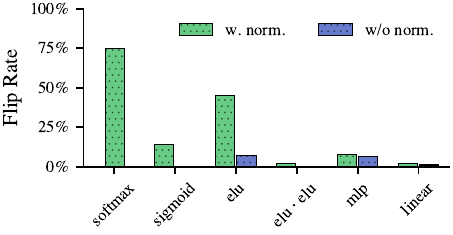}
    \caption{no MLP, LN}
  \end{subfigure}\hfill
  \begin{subfigure}[t]{0.46\textwidth}
    \centering
    \includegraphics[width=\textwidth]{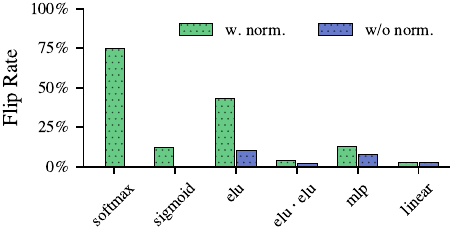}
    \caption{no MLP, no LN}
  \end{subfigure}

  \caption{Flip-rate plots for vocabulary size $|\mathcal{V}|=161$. Rows: MLP vs. no-MLP. Columns: LayerNorm vs. no-LayerNorm.}
  \label{fig:fliprate-vocab161}
\end{figure*}


\begin{figure*}[t]
  \centering

  \begin{subfigure}[t]{0.46\textwidth}
    \centering
    \includegraphics[width=\textwidth]{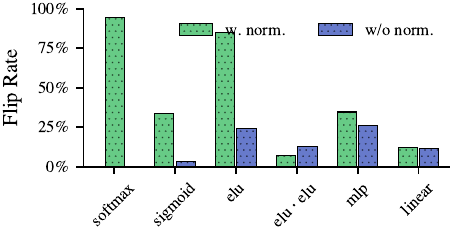}
    \caption{$H{=}1$, $L{=}1$}
  \end{subfigure}\hfill
  \begin{subfigure}[t]{0.46\textwidth}
    \centering
    \includegraphics[width=\textwidth]{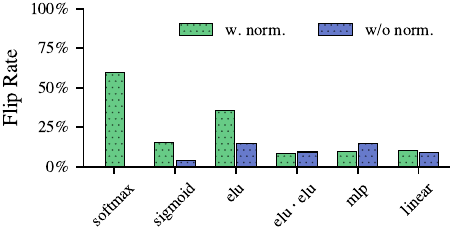}
    \caption{$H{=}1$, $L{=}2$}
  \end{subfigure}

  \vspace{2mm}

  \begin{subfigure}[t]{0.46\textwidth}
    \centering
    \includegraphics[width=\textwidth]{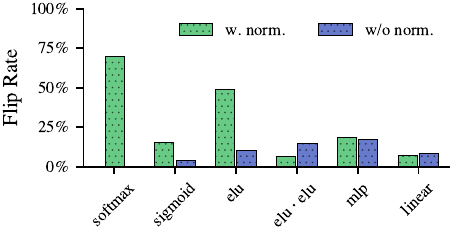}
    \caption{$H{=}2$, $L{=}1$}
  \end{subfigure}\hfill
  \begin{subfigure}[t]{0.46\textwidth}
    \centering
    \includegraphics[width=\textwidth]{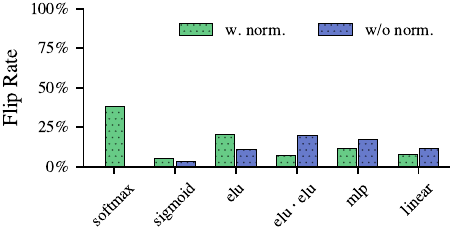}
    \caption{$H{=}2$, $L{=}2$}
  \end{subfigure}

  \vspace{2mm}

  \begin{subfigure}[t]{0.46\textwidth}
    \centering
    \includegraphics[width=\textwidth]{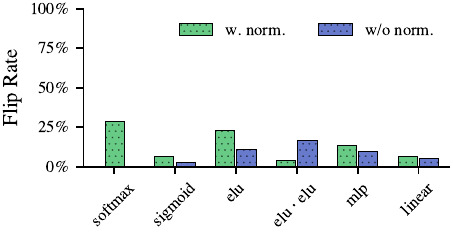}
    \caption{$H{=}4$, $L{=}1$}
  \end{subfigure}\hfill
  \begin{subfigure}[t]{0.46\textwidth}
    \centering
    \includegraphics[width=\textwidth]{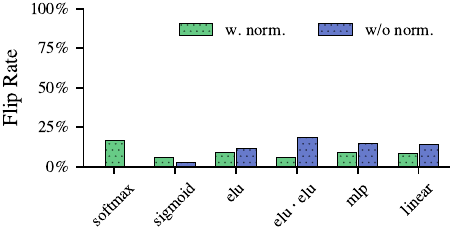}
    \caption{$H{=}4$, $L{=}2$}
  \end{subfigure}

  \caption{Flip-rate plots for $|\mathcal{V}|=41$ with no MLP and no LayerNorm. Rows: number of heads $H\in\{1,2,4\}$. Columns: number of layers $L\in\{1,2\}$. Increasing the number of heads or layers leads to the diminishing flip rate.}
  \label{fig:fliprate-vocab41-3x2-heads-layers-nomlp-noln}
\end{figure*}

\end{document}